\documentclass{article}
\PassOptionsToPackage{table,dvipsnames}{xcolor}
\PassOptionsToPackage{numbers, compress}{natbib}
\usepackage[main,final]{neurips_2026}

\usepackage[hidelinks,colorlinks=true,linkcolor=blue,citecolor=blue]{hyperref}
\usepackage{url}
\usepackage{float}
\usepackage{graphicx}
\usepackage{subcaption}
\usepackage{placeins}
\usepackage{wrapfig}
\usepackage{enumitem}
\usepackage{diagbox}
\usepackage{pgf}
\usepackage{microtype}
\usepackage[multiple]{footmisc}
\usepackage[textsize=scriptsize]{todonotes}
\usepackage{mathtools}
\usepackage{booktabs}
\usepackage{tabularx}
\usepackage{tikz}
\usetikzlibrary{matrix,fit,backgrounds,calc}
\usepackage{multirow}
\usepackage{physics}
\usepackage{array}
\usepackage{arydshln}
\usepackage{amsthm}
\usepackage{algorithm}
\usepackage{algpseudocode}

\usepackage{amsmath,amsfonts,bm,amssymb}

\def\1{\bm{1}}

\DeclareMathAlphabet{\mathsfit}{\encodingdefault}{\sfdefault}{m}{sl}
\SetMathAlphabet{\mathsfit}{bold}{\encodingdefault}{\sfdefault}{bx}{n}

\newtheorem{proposition}{Proposition}
\newtheorem{theorem}{Theorem}
\newtheorem{assumption}{Assumption}
\newtheorem{lemma}{Lemma}
\newtheorem{definition}{Definition}

\newcommand{\startpara}[1]{{\noindent{\bf #1.}}}

\definecolor{ao}{rgb}{0.0, 0.5, 0.0}
\definecolor{bittersweet}{rgb}{0.88235294, 0.67843137, 0.00392157}
\newcolumntype{C}[1]{>{\centering\arraybackslash}p{#1}}

\newenvironment{tightitems}{
  \vspace{-8pt}
  \begin{itemize}[leftmargin=10pt, labelsep=2pt, itemsep=-1pt]
}{
  \end{itemize}
  \vspace{-8pt}
}

 \usepackage{natbib}
\title{Latent Q-Barrier Shielding for Safe In-Context Reinforcement Learning}

\begin{document}
\author{
  Minjae Kwon$^{*\dagger}$ \\
  University of Virginia \\
  \texttt{mjkwon@virginia.edu} \\
  \And
  Amir Moeini$^{*\dagger}$ \\
  University of Virginia \\
  \texttt{amoeini@virginia.edu} \\
  \And
  Shangtong Zhang \\
  University of Virginia \\
  \texttt{shangtong@virginia.edu} \\
  \AND
  Lu Feng \\
  University of Virginia \\
  \texttt{lu.feng@virginia.edu}
}

\footnotetext[1]{$^{*}$ Equal contribution.}
\footnotetext[2]{$^{\dagger}$ Corresponding authors.}
\maketitle
\begin{abstract}
Safe in-context reinforcement learning (ICRL) adapts online from interaction history without test-time parameter updates while controlling episode cost under a safety budget. Under out-of-distribution (OOD) deployment shifts, pretraining-only safe ICRL can give poor reward-safety tradeoffs because the remaining budget affects behavior only through frozen policy conditioning, not an explicit action-level check against predicted future cost. 
We propose a latent Q-Barrier shield that learns a context representation, latent dynamics, and an ensemble cost critic before deployment. 
Without parameter updates, the shield infers context from history and filters or softly reweights candidate actions using the remaining budget and predicted future cost. We prove a conditional one-step barrier-margin bound over realized candidate sets, with explicit Bellman and latent-prediction error terms.
Across five safe ICRL benchmarks, the shield improves deployment-time reward-safety tradeoffs over a strong safe-ICRL baseline: after a short context window, it achieves higher return in four of five benchmarks while matching or lowering average episode cost in all five.
Code is available at \href{https://github.com/kmj1122/q-barrier}{q-barrier}.
\end{abstract}    
\section{Introduction}\label{sec:introduction}
Reinforcement learning (RL)~\citep{sutton2018reinforcement} agents deployed in the real world must manage safety constraints such as resource limits, physical damage, or constraint violations during deployment. Constrained Markov decision processes (CMDPs)~\citep{altman2021constrained} formalize this requirement by assigning each transition both a reward and a cost and constraining expected cumulative cost under a user-specified budget. Most safe RL methods enforce this constraint during training, for example through Lagrangian penalties or constrained policy optimization~\citep{achiam2017constrained, ray2019benchmarking}.
However, training-time constraint satisfaction need not generalize under deployment
shifts, as demonstrated by safety violations on unseen simulated patients in
diabetes control~\citep{pmlr-v306-kwon26c}.

In-context reinforcement learning (ICRL) provides a different adaptation mechanism: a pretrained agent adapts to new tasks at test time without parameter updates by conditioning its policy on an expanding interaction history~\citep{duan2016rl, xu2022prompting, laskin2023incontext, sinii2024incontext, polubarov2025vintix, liu2026scalable}. Recent work extends this idea to safe ICRL. In particular, \citet{moeini2026scared} formulate safe ICRL as a constrained problem across test episodes and propose SCARED, a Lagrangian pretraining method whose fixed-point analysis characterizes budget-respecting policies under its assumptions.

However, existing safe ICRL methods, such as SCARED, mainly encode safety through pretraining
objectives or cost-conditioning, rather than through an explicit action-level
check against the remaining budget.
A frozen ICRL policy may still assign probability to actions whose
predicted future cost exceeds the remaining budget, especially in generalization
settings where the task and safety structure must be inferred from context. We
add a budget-aware shielding layer that evaluates candidate actions against the
remaining budget using a latent cost barrier and a pessimistic cost-to-go critic.
The shield then filters or softly reweights the policy distribution without
deployment-time parameter updates, turning learned safety preferences into an
explicit action-level check.

This motivates the central question of the paper: \textit{can a learned barrier improve a frozen safe-ICRL policy's reward-safety tradeoff by turning the remaining budget into an explicit action-selection constraint?} We answer
this question through a Q-Barrier shielding method, a margin-based analysis, and
empirical evaluation on deployment generalization tasks.
In summary, our contributions are:
\begin{tightitems}
\item \textbf{Barrier-based runtime shielding for safe ICRL deployment.}
We introduce a deployment-time shield that reweights a frozen policy using the remaining safety budget and a learned cost-to-go estimate. The shield uses latent representations, latent dynamics, and an ensemble cost critic learned during training, with all parameters fixed at deployment.

\item \textbf{Conditional barrier-margin bound.}
We prove that for any selected action, the next-step latent-budget margin is lower-bounded by the current margin minus the critic's Bellman upper-bound error and a latent-prediction error term.

\item \textbf{Improved in-context adaptation reward-safety tradeoffs.}
Q-Barrier improves early return and lowers cost after a short context window in
four of five benchmarks, with one environment showing a conservative
lower-cost/lower-return tradeoff. In budget sweeps, it achieves the highest
cumulative return in four of five environments and satisfies the budget across
most levels, 
with violations mostly occurring at near-zero budgets, where even small costs can exceed the budget.
\end{tightitems}
 \section{Background}\label{sec:background}
\textbf{Markov Decision Processes.} We model the interaction with an environment as a finite-horizon Markov decision process (MDP \citep{puterman2014markov}).
An MDP is defined by a state space $\mathcal{S}$, an action space $\mathcal{A}$, a reward function $r : \mathcal{S} \times \mathcal{A} \to \mathbb{R}$, a transition kernel $p(\cdot \mid s,a)$, where for each state-action pair $(s,a)$, $p(\cdot \mid s,a)$ denotes a probability distribution over next states in $\mathcal{S}$, an initial state distribution $p_0$ over $\mathcal{S}$, and a horizon length $T$. We denote by $T_k$ the terminal timestep of episode $k$. 
An agent follows a stochastic policy $\pi$, modeled as a Markov kernel from
states to actions: for each $s\in\mathcal S$, $\pi(\cdot\mid s)$ is a
probability distribution over $\mathcal A$. 
When $\mathcal A$ is discrete, $\pi(a\mid s)$ denotes the probability mass of
action $a$. When $\mathcal A$ is continuous, $\pi(a\mid s)$ denotes the action
density, when it exists, which integrates to one.
After starting from
$S_0\sim p_0$, at time step $t$, the agent samples
$A_t\sim\pi(\cdot\mid S_t)$, observes reward
$R_{t+1}\doteq r(S_t,A_t)$, and transitions to
$S_{t+1}\sim p(\cdot\mid S_t,A_t)$.
We use $k$ to index episodes and $t$ to index time steps within an episode.
For any episode $\tau$, we define its return as $G(\tau) \doteq \sum_{t=1}^T R_t$.
The performance of a policy $\pi$ is measured by the expected total rewards $J(\pi) \doteq \mathbb{E}_{\tau \sim \pi}[G(\tau)]$.

\textbf{Constrained Markov Decision Processes.}
In addition to the reward function $r$, CMDPs \citep{altman2021constrained} involve a cost function $c : \mathcal{S} \times \mathcal{A} \to \mathbb{R}^{+}$ with an associated user-given budget $\delta$.
At each time step $t$, after taking an action $A_t$ at a state $S_t$, the agent receives a cost $C_{t+1} \doteq c(S_t, A_t)$.
An episode in a CMDP therefore becomes $\tau = (S_0, A_0, R_1, C_1, S_1, \dots, S_{T-1}, A_{T-1}, R_T, C_T).$ We denote the total cost of an episode $\tau$ by $G_c(\tau) \doteq \sum_{t=1}^T C_t$, which should, in expectation, remain below the budget $\delta$.
The return-to-go (RTG) is $G_t(\tau) \doteq \sum_{i=t+1}^T R_i$ and the cost-to-go (CTG) is $G_{c,t}(\tau) \doteq \sum_{i=t+1}^T C_i$. 

\textbf{In-Context Reinforcement Learning.}
ICRL~\citep{moeini2025survey} trains an agent across a distribution of tasks so that, at test time, it can adapt without parameter updates by conditioning on an expanding interaction history.
Let $\theta_*$ denote the pretrained parameters.
During deployment, $\theta_*$ is fixed and the policy is written as $\pi_{\theta_*}(\cdot \mid S_t^k, H_t^k),$ where $H_t^k$ is the context available at time $t$ in episode $k$.
A context consists of previous episodes $\{\tau_{j}\}_{j=1, \dots k-1}$ and the current episode prefix $\tau_k^t$, i.e., $H_t^k \doteq (\tau_1,\ldots,\tau_{k-1},\tau_k^{t}).$
In safe ICRL~\citep{moeini2026scared}, each task is a CMDP, so the context includes both rewards and costs: 
$\tau_k^{t} \doteq (S_0^k,A_0^k,R_1^k,C_1^k,S_1^k,\ldots, S_{t-1}^k,A_{t-1}^k,R_t^k,C_t^k).$
We define the remaining cost budget as $B_t^k \doteq \delta - \sum_{i=1}^{t} C_i^k, ~B_0^k=\delta .$
Adaptation occurs as $\pi_{\theta_*}(\cdot \mid S_t^k,H_t^k)$ changes with the growing context, while the policy parameters remain fixed.

\subsection{Problem Formulation}\label{sec:problem}
We formulate safe ICRL deployment over $K$ test episodes. All parameters are learned before deployment and fixed at test time. The deployed shielded policy $\pi$ uses the current state $S_t^k$, context $H_t^k$, and remaining budget $B_t^k$ for action selection. Our objective is:
\begin{equation}
\label{eq:safe_icrl_objective}
\textstyle
\max_{\pi}\;
\sum_{k=1}^K \mathbb{E}_{\pi}[G(\tau_k)]
\quad
\text{s.t.}\quad
\frac{1}{K}\sum_{k=1}^K \mathbb{E}_{\pi}[G_c(\tau_k)] \le \delta.
\end{equation}
The shield changes actions online to improve reward-safety tradeoffs during deployment. 
Theorem~\ref{thm:barrier_margin_propagation} gives a conditional bound on how the barrier margin propagates one step forward under the learned critic. Proposition~\ref{prop:episode_budget_bound} extends it to an episode-level budget bound with explicit error terms.

 \section{Approach}\label{sec:approach}
We present the method in three parts. First, we define the latent state and learned safety model used for runtime safety estimation. Second, we introduce a Q-function barrier that uses the learned cost-to-go estimate and the remaining safety budget to filter actions online. Third, we describe how the latent representation, dynamics model, and cost critic are trained to support shielding at deployment.

\subsection{Latent State and Projection Spaces}\label{subsec:latent}
The shield uses two views of the same in-context state: a policy view for proposing actions and a world view for evaluating safety. Let $\mathcal{H}$ denote the context space and let $E_\phi:\mathcal{H}\times\mathcal{S}\to \mathcal{Z}\subseteq \mathbb{R}^{d_z}$ be the shared encoder. Given context $H_t^k$ and state $S_t^k$, the shared latent state is $Z_t^k \doteq E_\phi(H_t^k,S_t^k)\in\mathcal{Z}.$ From this shared latent, we define two projection heads with a common projected dimension $d_m$: $g_\omega^{\mathrm{world}}:\mathcal{Z}\to\mathcal{Z}^w\subseteq\mathbb{R}^{d_m}, ~ g_\psi^{\mathrm{policy}}:\mathcal{Z}\to\mathcal{Z}^p\subseteq\mathbb{R}^{d_m}.$
The corresponding projected latents are $Z_t^{w,k}\doteq g_\omega^{\mathrm{world}}(Z_t^k), ~ Z_t^{p,k}\doteq g_\psi^{\mathrm{policy}}(Z_t^k).$ The world-projected latent $Z_t^{w,k}$ is the safety and modeling space used by the latent dynamics model, cost critic, and runtime barrier. The policy-projected latent $Z_t^{p,k}$ is the policy-side representation used by the base policy to form its action distribution. The shield proposes actions through the policy view and evaluates their budget risk through the world view. 
On $\mathcal{Z}^w$, we learn a probabilistic latent transition model
$p_z(\cdot \mid Z_t^{w,k},A_t^k)$ and an ensemble of cost critics
$\{\hat Q_{C, i}\}_{i=1}^M$. The barrier uses the pessimistic
aggregate $\hat Q_C^+(Z,A)
\doteq
\max_{i\in\{1,\ldots,M\}}\hat Q_{C, i}(Z,A),$
which reduces the risk of underestimating future cost. 
We write
$f_z(Z_t^{w,k},A_t^k)
\doteq
\mathbb{E}_{p_z}[Z_{t+1}^{w,k}\mid Z_t^{w,k},A_t^k]$
for the predictive mean. The critic $\hat Q_C^+$ estimates cost-to-go in
the same units as the remaining budget, and $f_z$ supports representation learning and analysis, not
runtime action selection. All modules are learned during training and frozen at test time.

\subsection{Q-Function Barrier and Shield Policy}\label{subsec:barrier}
The shield maps a base policy into a budget-aware action distribution. At each decision time $t$ in episode $k$, the shield evaluates a finite candidate action set
$\mathcal{A}_{t, k}^{\mathrm{cand}}\subseteq\mathcal{A}$.
In discrete-action settings, $\mathcal{A}_{t, k}^{\mathrm{cand}} = \mathcal{A}$, and in continuous-action settings, it is the finite set of candidate actions sampled from $\pi_\theta$. At state $Z_t^{w,k}$ with remaining budget $B_t^k$, actions whose predicted cost-to-go exceeds the available budget are discouraged.
This motivates a learned barrier signal derived from the cost critic: the value is positive when the remaining budget suffices to cover the predicted cost-to-go and negative when the action is predicted to overspend the budget.
Because the barrier depends on the current latent state and learned cost-to-go estimate, it adapts online as context accumulates, without changing any parameters.

\begin{definition}[Q-Function Barrier]
\label{def:q_barrier}
For $Z \in \mathcal{Z}^w$, define the candidate-set cost value
$\hat V^{+, k}_{C, t}(Z)
\doteq
\min_{A \in \mathcal{A}_{t, k}^{\mathrm{cand}}} \hat Q^{+}_C(Z, A).$
The state- and state-action barriers are
\begin{equation}
b^k_{V, t}(Z_t^{w,k}, B_t^k)
\doteq
B_t^k - \hat V^{+, k}_{C, t}(Z_t^{w,k}),
\qquad
b^k_{Q, t}(Z_t^{w,k}, B_t^k, A_t^k)
\doteq
B_t^k - \hat Q^{+}_C(Z_t^{w,k}, A_t^k).
\label{eq:q_barrier}
\end{equation}
\end{definition}
The state barrier $b^k_{V, t}\ge 0$ means that at least one candidate action has predicted cost-to-go no larger than the remaining budget. The action barrier $b^k_{Q, t}$ evaluates whether a specific action is predicted to stay within the remaining budget under $\hat Q^+_C$. Because this margin is estimated by learned models, we use it as a soft risk signal rather than a hard certificate. The primary variant is a \emph{soft shield} that reweights $\pi_\theta$ over $\mathcal{A}_{t, k}^{\mathrm{cand}}$, preserving policy support while penalizing candidates whose predicted cost exceeds the remaining budget. 
We write $[x]_+ \doteq \max\{x,0\}$.
For each candidate action, define the soft-shield score as 
$$w_t(A)
\doteq
\rho_t(A)
\exp\!\left(
-\left[-b^k_{Q, t}(Z_t^{w,k},B_t^k,A)\right]_+
\right),
\text{~~where~~}$$
$$
\rho_t(A)
=
\begin{cases}
\pi_\theta(A\mid Z_t^{p,k}), 
& \text{if the discrete action space is enumerated},\\
1,
& \text{if } A\in\mathcal A_{t,k}^{\mathrm{cand}}
\text{ is sampled from } \pi_\theta(\cdot\mid Z_t^{p,k}).
\end{cases}
$$
The soft-shield distribution over the finite candidate set is
\begin{equation}
\pi_{\mathrm{soft}}(A \mid S^k_t, H_t^{k}, B^k_t)
=
\frac{w_t(A)}
{\sum_{A' \in \mathcal{A}_{t, k}^{\mathrm{cand}}} w_t(A')},
\quad
A \in \mathcal{A}_{t, k}^{\mathrm{cand}}.
\label{eq:soft_shield}
\end{equation}
\textbf{Remark.}
We use the soft shield as the default because it preserves support over sampled
candidates and avoids sensitive hard-threshold decisions. Hard filtering can be
sensitive to model or critic error: when no candidate is considered as feasible,
it falls back to the lowest-predicted-cost action. Empirically, both variants
achieve similar reward--safety profiles in most environments (Appendix~\ref{app:ablation_shield_type}).
Unlike hard filtering in Appendix~\ref{app:hard_shield}, it can still choose an action with
$b^k_{Q, t}<0$. In continuous control, candidates are already drawn from the base
policy, so we omit an additional policy-density factor in the finite-candidate
weights to avoid double-counting the proposal.

\subsection{Training Objective and Asymmetric Gradient Routing}
\label{subsec:training}
The shield uses $Z_t^{p,k}$ for action proposals and $Z_t^{w,k}$ for latent prediction and cost estimation. All modules are trained before deployment and then frozen. To reduce interference between policy learning and world modeling, we use asymmetric gradient routing: the world-model loss updates the shared encoder, while stop-gradient alignment losses coordinate the projection heads without directly updating the encoder.

\textbf{Representation-shaping loss.}
We train an auxiliary one-step world model on the world-projected latent space. The loss is
$\mathcal{L}_{\mathrm{wm}}
=
\mathbb{E}~\![
-d_m^{-1}\log p_z~\!(Z_{t+1}^{w,k}\mid Z_t^{w,k}, A_t^k)
]
+
\mathbb{E}~\![
(\hat R^k_{t+1}-R^k_{t+1})^2
]
+
\mathbb{E}~\![
\sum_{j=1}^{2}(\hat C^k_{j,t+1}-C^k_{t+1})^2
],$
where $d_m$ is the common output dimension of the policy and world projection heads, and $\hat R$ and $\hat C_j$ are the outputs of the reward head and the two immediate-cost heads. This loss trains the encoder and world projection by encouraging one-step predictability of latent dynamics, reward, and cost.

\textbf{Structural alignment losses.}
The losses $\mathcal{L}_{\text{distill}}$ and $\mathcal{L}_{\text{conj}}$ align the policy and world projections through detached encoder features, with the deployment discount $\gamma_*$ defined below: $\mathcal{L}_{\text{distill}}^{\text{sg}} =
\frac{1}{d_m}\left\|
g_{\psi}^{\text{policy}}(\text{sg}(Z^k_t))
-
\text{sg}(g_{\omega}^{\text{world}}(Z^k_t))
\right\|_2^2,$ and
$$
\mathcal{L}_{\text{conj}}^{\text{sg}} =
\frac{1}{d_m}\left\|
\bigl(
g_{\psi}^{\text{policy}}(\text{sg}(Z^k_{t+1}))
-
\gamma_* g_{\psi}^{\text{policy}}(\text{sg}(Z^k_t))
\bigr)
-
\text{sg}\bigl(
g_{\omega}^{\text{world}}(Z^k_{t+1})
-
\gamma_* g_{\omega}^{\text{world}}(Z^k_t)
\bigr)
\right\|_2^2.
$$

We use AMAGO's actor--critic backbone~\citep{grigsby2023amago}, which trains heads at multiple discount factors, and select
the deployment discount $\gamma_*=0.999$ for rollout and evaluation. The conjugacy loss uses the same coefficient to match temporal changes between
the policy and world projections. We augment this backbone with SCARED's safety objectives~\citep{moeini2026scared}.
The actor combines reward improvement, advantage-filtered behavioral cloning, and a masked cost penalty, while the critics
minimize reward and cost TD errors. The principal training terms form the core objective
\begin{equation}
\label{eqn:q-barrier loss}
\begin{aligned}
\mathcal{L}_{\text{core}} ={}&
\mathcal{L}_{\text{actor}}
+
\lambda_{\mathrm{critic}} \mathcal{L}_{\text{critic}}
+
\lambda_{\text{wm}}\mathcal{L}_{\text{wm}}
+
\lambda_{\text{dist}}\mathcal{L}_{\text{distill}}^{\text{sg}}
+
\lambda_{\text{conj}}\mathcal{L}_{\text{conj}}^{\text{sg}},
\end{aligned}
\end{equation}
where ${\text{sg}}$ denotes stop-gradient. Appendix~\ref{app:loss_details} details these objectives, their weights, and gradient routing.
Asymmetric gradient routing encourages $Z^w$ to support latent prediction and cost estimation while preserving policy-relevant features in $Z^p$.

\section{Barrier Margin Propagation}\label{sec:theory}
We decompose the next-step barrier value into three quantities: the current
barrier margin, the one-step latent prediction error, and the Bellman
upper-bound error of the learned cost critic. 
For any fixed next-step candidate set
$\mathcal A_{t+1,k}^{\mathrm{cand}}$, define $\hat V_{C,t+1}^{+,k}(Z)
\doteq
\min_{A'\in\mathcal A_{t+1,k}^{\mathrm{cand}}}
\hat Q_C^+(Z,A'),$ with $\hat V_{C,T_k}^{+,k}=0$.
In continuous-action domains, $\mathcal A_{t+1,k}^{\mathrm{cand}}$ is the finite
set of candidate actions sampled by the shield, so all statements are
conditional on the sampled candidate set.

We first assume that the learned latent dynamics model predicts the next latent state with small one-step error, motivated by the training objective in Section~\ref{subsec:training}.
\begin{assumption}[Approximate Latent Dynamics]
\label{ass:approx_projected_latent_grounding}
Along the trajectories considered, the learned latent transition mean $f_z$ satisfies $\left\|
f_z(Z_t^{w,k}, A_t^{k}) - Z_{t+1}^{w,k}
\right\|_2
\le
\varepsilon_{\mathrm{pred}}$
for all episodes $k$ and timesteps $t$ .
\end{assumption}

We also assume that the learned cost critic changes smoothly when the latent state changes slightly.
\begin{assumption}[Latent Q-Regularity]
\label{ass:q_regular}
For $A$ in the candidate action set, the pessimistic cost critic
$\hat Q_C^+:\mathcal{Z}^{w}\times\mathcal{A}\to\mathbb{R}$
is uniformly Lipschitz. That is, there exists
$L_Q > 0$ such that
\begin{equation}
\left|
\hat Q^+_C(Z,A)-\hat Q^+_C(Z',A)
\right|
\le
L_Q \|Z-Z'\|_2
\quad
\text{for all } Z,Z' \in \mathcal{Z}^{w}.
\label{eq:q_lipschitz}
\end{equation}
\end{assumption}

We next isolate the critic's one-step Bellman error as an explicit residual term.
\begin{definition}[Bellman Upper-Bound Error]
\label{def:bellman_error}
For transition $(Z_t^{w,k},A_t^k,C_{t+1}^k)$, define
\begin{equation}
\Delta_{\mathrm{bell},t}^{+,k}
\doteq
\left[
C_{t+1}^k
+
\hat V_{C,t+1}^{+,k}
\!\left(f_z(Z_t^{w,k},A_t^k)\right)
-
\hat Q_C^+(Z_t^{w,k},A_t^k)
\right]_+ .
\label{eq:bellman_error}
\end{equation}
\end{definition}
This value measures how much the critic violates the one-step Bellman upper bound.
$\Delta_{\mathrm{bell},t}^{+, k}=0$ when the critic holds the one-step upper-bound condition
$\hat Q^+_C(Z_t^{w,k},A_t^k)\ge C_{t+1}^k+\hat V^{+, k}_{C, t+1}(f_z(Z_t^{w,k},A_t^k))$.
The main theorem below shows how the barrier margin propagates forward by losing Bellman error and model-prediction error.
\begin{theorem}[Barrier Margin Propagation on the Latent Space]
\label{thm:barrier_margin_propagation}
Under Assumptions~\ref{ass:approx_projected_latent_grounding}
and~\ref{ass:q_regular}, for any selected action $A_t^k$,
$
b_{V,t+1}^k(Z_{t+1}^{w,k},B_{t+1}^k)
\ge
b_{Q,t}^k(Z_t^{w,k},B_t^k,A_t^k)
-
\Delta_{\mathrm{bell},t}^{+,k}
-
L_Q\varepsilon_{\mathrm{pred}} .
$
\end{theorem}

The theorem implies an approximate one-step continuation property: if
$b^k_{Q, t}(Z_t^{w,k},B_t^k,A_t^k)\ge 0$, then $b^k_{V, t+1}(Z_{t+1}^{w,k},B_{t+1}^k)
\ge
-\Delta_{\mathrm{bell},t}^{+, k}
-
L_Q\varepsilon_{\mathrm{pred}}.$
Thus, the next latent-budget state has a lower-bounded margin under the learned critic, up to Bellman and prediction
errors. With enough positive barrier margin, it has a nonnegative continuation
margin. This is a conditional statement about the selected action, not a
closed-loop guarantee for the soft shield, which may still sample
actions with $b^k_{Q, t}<0$.
See Appendix~\ref{apx:proof} for proofs, and Appendix~\ref{apx:episode_bound} for an episode-level budget bound.
 \section{Experiments}\label{sec:experiments}
We evaluate our method on the safe ICRL benchmark~\citep{moeini2026scared}, comparing against safe ICRL methods~\citep{moeini2026scared} and safe meta-RL baselines augmented with a context encoder~\citep{finn2017model,xu2025efficient}. Our experiments are designed to answer two questions: (i) whether shielding improves the reward-safety tradeoff during in-context adaptation without parameter updates, and (ii) how this tradeoff changes as the safety budget varies.

\subsection{Baselines}
We compare against two baseline groups. The first consists of parameter-update-free safe adaptation methods, which adapt through context at test time without gradient updates. This group includes Safe AD and SCARED, which differ mainly in how safety is learned during pretraining. The second consists of safe meta-RL methods, which adapt through per-task parameter updates at test time.

\textbf{Safe Algorithm Distillation (Safe AD).}
Following algorithm distillation~\citep{laskin2023incontext}, Safe AD~\citep{moeini2026scared} trains a history-conditioned policy on logged learning trajectories generated by safe RL algorithms. At test time, the policy conditions on the interaction history together with RTG and CTG targets, without any parameter updates.

\textbf{SCARED.}
This is the state-of-the-art safe ICRL baseline trained with an exact-penalty dual objective to improve reward-safety tradeoffs through in-context adaptation~\citep{moeini2026scared}. We use SCARED as the base policy for our shield.

\textbf{Safe Meta-RL.}
We compare against gradient-based safe meta-RL methods: MAML with penalty~\citep{finn2017model} and SafeMeta~\citep{xu2025efficient}. MAML with penalty learns a cost-regularized initialization, while SafeMeta
meta-trains a closed-form one-step safe adaptation rule. Unlike in-context baselines, both
adapt through parameter updates rather than conditioning on cross-episode
histories.

\subsection{Experimental Setup}

\textbf{Safe ICRL Benchmarks.}
We evaluate our approach on five environments from three safe ICRL benchmarks~\citep{moeini2026scared} studying two complementary generalization regimes: \emph{structural OOD} tasks and \emph{unseen ID} tasks. For structural OOD generalization, we use SafeDarkRoom and SafeDarkMujoco (Point, Car), where test tasks involve shifted goal and obstacle configurations under partial observability: Range-based sensing (e.g., lidar) is unavailable, so the agent cannot directly observe goal or obstacle locations and must adapt from in-context
history. For unseen ID generalization, we use SafeVelocity (HalfCheetah, Ant), where each task is specified by a target velocity. We train on a subset of target velocities from a fixed range and evaluate on held-out velocities from the same range. See Appendix~\ref{app:env_details} for more details.

\textbf{Training Setup.}
For the structural OOD environments, all algorithms are trained on source tasks
with center-oriented obstacles and goals and evaluated on 100 OOD test tasks with edge-oriented obstacle and goal distributions.
For the unseen-ID environments, target-velocity multipliers are sampled uniformly
from $[0.5,2.0]$ during training. We evaluate each method on 100 held-out
contexts sampled from the same range, keeping the target velocity fixed across
the in-context episode adaptation.

Safe AD follows the training protocol of~\citep{moeini2026scared}, using
trajectories generated by PPO-Lag~\citep{ray2019benchmarking} and the training hyperparameters.
We use SCARED as the base policy and keep its original pretraining budgets:
3K epochs for SafeDarkRoom and SafeDarkMujoco, and 2K for SafeVelocity.
An epoch is one collection-update cycle, consisting of 1.5K environment steps
per training environment followed by 1K optimization batches. The
shield-supporting modules in Section~\ref{sec:approach} are trained during this
same pretraining phase, with no additional epochs. At test time, Q-Barrier shield reweights candidate actions using Algorithm~\ref{alg:runtime_shielding}, without parameter updates. Safe meta-RL baselines are trained for up to 15K meta-training steps on each environment, with most converging earlier.
For Safe AD and the safe meta-RL baselines, we perform hyperparameter optimization to ensure a fair comparison. 
All loss weights in Equation~\eqref{eqn:q-barrier loss} for our
shield-supporting modules are fixed across environments. See Appendix~\ref{app:training_details} for full hyperparameters
and compute details. 

\textbf{Evaluation Setup.}
To answer the two research questions, we use a consistent target-conditioning setup across methods.
For SCARED and our Q-Barrier, we evaluate robustness to safety budgets by conditioning on CTG values sampled uniformly from $[1,15]$ in SafeDarkRoom, $[10,50]$ in SafeDarkMujoco, and $[0,5]$ in SafeVelocity, and we report performance averaged over 100 distinct CTG targets. 
Q-Barrier uses the soft shield with a candidate-set size of $N_s=8$ in
continuous-action environments. 
The same CTG intervals and number of targets are used for SafeMeta, MAML with
penalty, and Safe AD. For Safe AD, following~\citep{moeini2026scared}, we use
paired RTG-CTG conditioning: each CTG target sampled from the same interval is
paired with a corresponding RTG target.
In SafeDarkRoom and SafeDarkMujoco, RTG is set as
$\max\{0.1,\mathrm{CTG}/\mathrm{CTG}_{\max}\}$ with
$\mathrm{CTG}_{\max}=15$ and $50$, respectively. In the unseen-ID
tasks, RTG and CTG are jointly scaled from $0.5$ to $\mathrm{RTG}_{\max}=200$, and $\mathrm{CTG}_{\max}=5$ for both.

\subsection{Results Analysis}

\textbf{RQ1: Does shielding improve the reward-safety tradeoff during in-context adaptation without parameter updates?}

\begin{figure*}[t]
    \centering
    \includegraphics[width=1.0\linewidth]{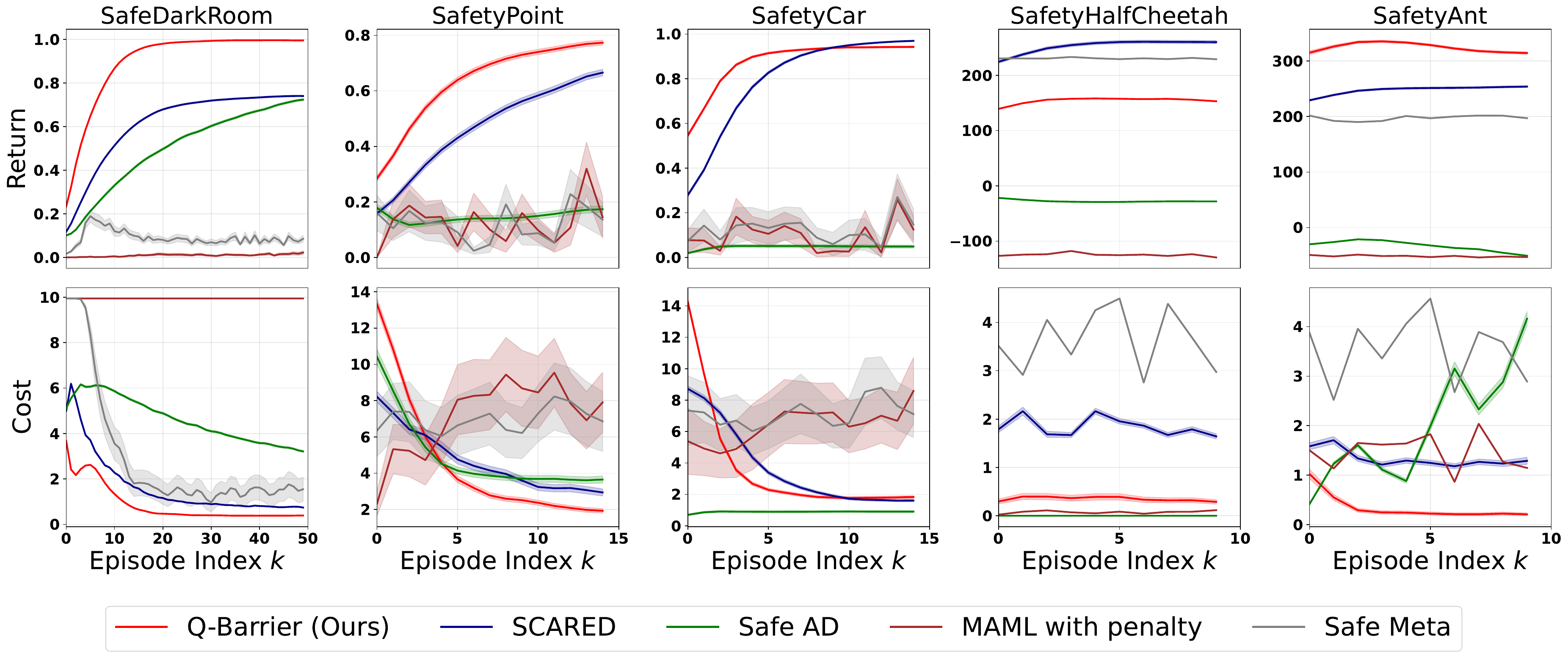}
    \caption{
    \textbf{In-context adaptation dynamics (RQ1).}
      Per-episode return (top) and cost (bottom) over in-context episode index $k$ on test-time evaluation tasks. 
     Curves show means over 100 evaluation tasks, with shaded regions denoting standard error of the mean.
    Q-Barrier generally improves adaptation by reaching high return earlier and reducing cost
after a short context window, while SafetyHalfCheetah shows a conservative Pareto point
with lower cost and lower return.
    }
    \label{fig:rq1}
\end{figure*}

We compare Q-Barrier and the baselines over in-context episodes on
test tasks, measuring how return and cost evolve as context accumulates
(Figure~\ref{fig:rq1}). Since Q-Barrier uses SCARED as its base policy, the most
direct comparison is against SCARED. Safe AD, Safe Meta, and MAML with penalty provide additional safe ICRL and safe
meta-RL baselines. During adaptation, these baselines fail to reduce cost consistently or become overly conservative sacrificing
return.

Q-Barrier generally improves the in-context adaptation by reaching high return
earlier while reducing cost after a short context window. In
SafeDarkRoom and SafetyAnt, the improvement is immediate. 
In SafeDarkRoom, Q-Barrier achieves mean return
$0.9$ compared with $0.62$ for SCARED, a $46\%$ gain, while reducing mean
cost from $1.66$ to $0.87$, a $47.2\%$ drop. 
In SafetyAnt, Q-Barrier
achieves mean return $324.36$ compared with $247.8$ for SCARED, a $30.9\%$
gain, while reducing mean cost from $1.34$ to $0.34$, a $74.3\%$ drop.

DarkMujoco navigation tasks show a different but informative pattern:
Q-Barrier initially encounters higher cost, but this early cost comes with
higher return. 
In SafetyPoint and SafetyCar, Q-Barrier first becomes lower-cost than SCARED at episodes
$k=3$ and $k=2$, respectively. Averaged over the remaining adaptation, 
Q-Barrier reduces mean cost from $4.0$ to $2.99$ in SafetyPoint, a $25.3\%$
drop, and from $2.95$ to $2.37$ in SafetyCar, a $19.7\%$ drop.
At the same time, Q-Barrier maintains an early return advantage in both tasks.
For example, in SafetyCar, Q-Barrier exceeds return $0.9$ by episode $k=5$,
whereas SCARED reaches this level around episode $k=7$. 

SafetyHalfCheetah is the main conservative case. During the full in-context
adaptation, Q-Barrier consistently lowers cost but also lowers return
relative to SCARED. Averaged over the adaptation, it reduces mean cost
from $1.84$ to $0.35$, a $81.0\%$ drop, while mean return decreases from
$252.75$ to $154.26$, a $39.0\%$ drop. This places Q-Barrier at a more
conservative point on the reward-cost frontier. 
The results in Appendix~\ref{app:diagnostics} provide a plausible
explanation: SafetyHalfCheetah has the largest latent prediction error among the
evaluated environments, which can make the learned barrier more conservative.

\textit{Takeaway:  Q-Barrier improves the reward-safety tradeoff
during in-context adaptation, with SafetyHalfCheetah showing a conservative
Pareto point that trades lower return for lower cost.}

\textbf{RQ2: How does the reward-safety tradeoff change across safety budgets?}

We next evaluate the objective in Equation~\eqref{eq:safe_icrl_objective}
by sweeping the cost budget and measuring cumulative return together with average
episode cost over the in-context adaptation. 
\begin{figure*}[t]
    \centering
    \includegraphics[width=1.0\linewidth]{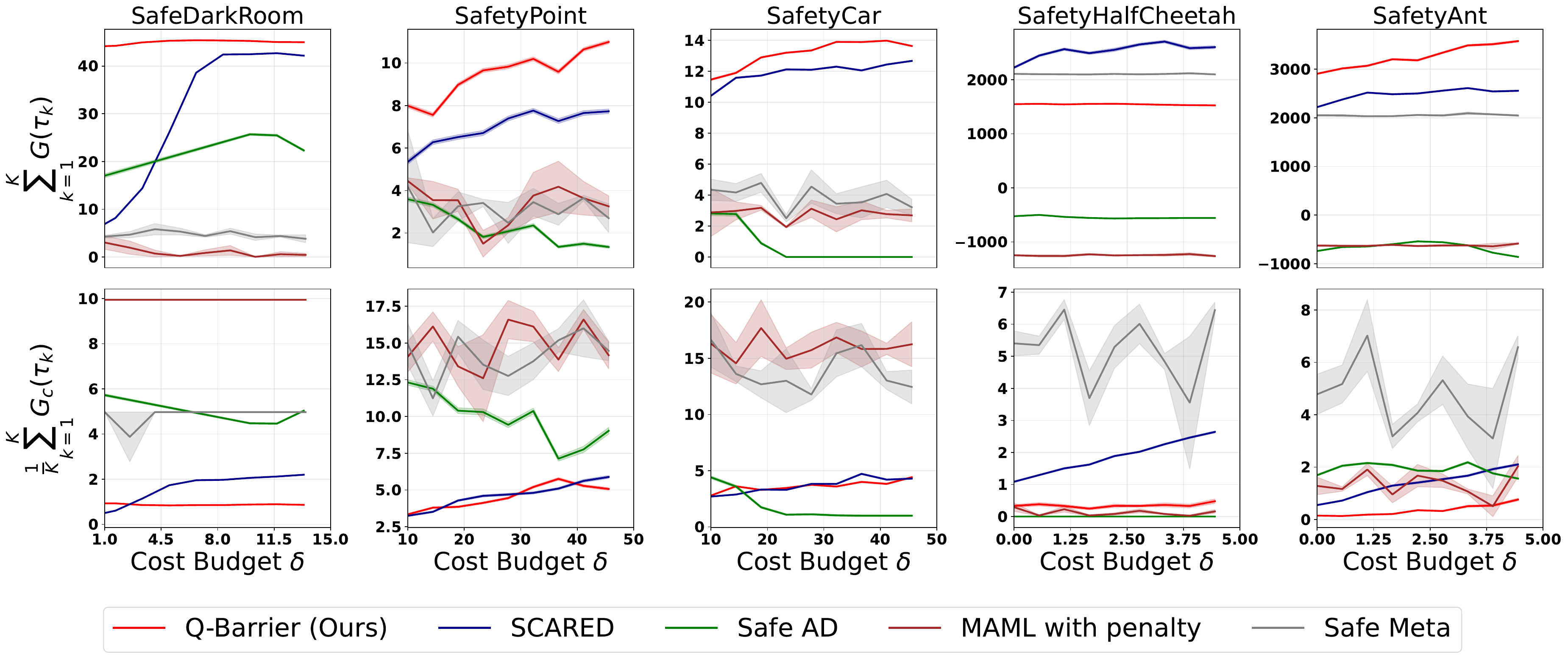}
    \caption{
    \textbf{Reward-cost tradeoffs under budget variation (RQ2).}
        Cumulative reward (top) and average episode cost (bottom) over varying
    cost budgets on test-time evaluation tasks. Across different budget levels, Q-Barrier achieves the highest cumulative return in
    four of five environments while maintaining budget
compliance in most cases, with violations mostly limited to near-zero budgets.
    Shaded
regions denote standard errors of the mean over 100 evaluation tasks.
    }
    \label{fig:rq2}
\end{figure*}

Figure~\ref{fig:rq2} shows that Q-Barrier achieves the highest cumulative
return in four of five environments across different budgets: SafeDarkRoom, SafetyPoint, SafetyCar, and
SafetyAnt. It also satisfies the average episode-cost constraint across most
budget settings in the sweep. Violations occur only at near-zero
budgets, where even small residual costs exceed the requested budget. 
This is the most difficult setting for Q-Barrier, since small model or
critic errors can produce a violation when the allowed cost is near zero.

Compared with SCARED base policy, Q-Barrier improves average
cumulative return over budget levels by $54.4\%$ in SafeDarkRoom, $36.3\%$ in
SafetyPoint, $10.1\%$ in SafetyCar, and $31.0\%$ in SafetyAnt. Q-Barrier also reduces average episode
cost relative to SCARED by $44.1\%$ in SafeDarkRoom, $2.1\%$ in SafetyPoint,
$1.0\%$ in SafetyCar, and $74.2\%$ in SafetyAnt. 
Thus, in most environments, Q-Barrier improves return without increasing average
cost, and in SafeDarkRoom and SafetyAnt it improves both return and cost by large
margins. 

SafeDarkRoom demonstrates the benefit of explicit budget-aware action selection.
SCARED becomes conservative as the budget varies and does not fully convert
available budget into return. 
In particular, given budget near $1$, Q-Barrier exceeds return $40$ while satisfying budget constraints, but SCARED stays under return $10$.

SafetyHalfCheetah is the main exception. Q-Barrier reduces average episode cost
by $81.3\%$ relative to SCARED, but its cumulative return is $39.1\%$ lower.
Safe Meta also achieves higher cumulative return than Q-Barrier, but Safe Meta often
violates the cost budget by a large margin (e.g., average cost around $6$ when cost budget is $1$).
This is consistent with the conservative behavior observed in RQ1, which can be explained by the largest latent prediction error.

\textit{Takeaway: Q-Barrier improves the reward-safety frontier across
budget levels in most environments, often increasing return while matching or
reducing cost. SafetyHalfCheetah remains the conservative case, trading lower
return for substantially lower cost.}

\textbf{Ablation Studies.}
We include ablations and diagnostics for the main shielding design choices.
Appendix~\ref{app:additional_ablations} adds auxiliary-loss checks, a cost-penalty comparison, executed-action margins, candidate feasibility across $N_s$, the training context length, and longer adaptation on Point and Car.
First, we compare the soft and hard rules in Algorithm~\ref{alg:runtime_shielding}, where hard truncation restricts candidates to those with nonnegative deployment margin (Appendix~\ref{app:hard_shield}).
The soft shield matches the hard shield's reward-safety tradeoffs while
preserving broader action support during early adaptation. Since hard filtering
provides marginal and non-uniform cost reductions, we use soft reweighting
in all main experiments (Appendix~\ref{app:ablation_shield_type}).
Second, we study how the number of sampled candidate actions $N_s$ affects shielding in continuous-control environments, sweeping $N_s \in \{4, 8, 16, 32\}$.
Performance is stable across $N_s$ in SafeDarkMujoco, while SafeVelocity is more
affected by $N_s$: larger candidate sets raise both return and cost (Appendix~\ref{app:ablation_sampling}).
Appendix~\ref{app:diagnostics} reports empirical diagnostics related to the barrier-margin bound (Theorem~\ref{thm:barrier_margin_propagation}), including latent prediction error, critic sensitivity, and discounted Bellman-residual proxies, consistent with the conservative shielding observed in SafetyHalfCheetah.
Appendix~\ref{app:algorithms} provides pseudocode for the training and runtime shielding procedures, Appendix~\ref{app:env_details} describes the benchmark environments and OOD generation protocol, and Appendix~\ref{app:training_details} lists full hyperparameters for all methods.
 \section{Related Work}\label{apx:related_works}
\startpara{Safe Reinforcement Learning}
Safe reinforcement learning is commonly formalized through constrained Markov
decision processes (CMDPs), where policies maximize reward subject to expected
cost constraints. Existing methods include constrained policy optimization,
Lagrangian and primal-dual algorithms, projection-based policy optimization,
recovery policies, model-based safe RL, and sample-efficient constraint-handling
methods
\citep{achiam2017constrained,wachi2020safe,ray2019benchmarking,
zhang2020focops,yang2020PCPO,yang2022cup,gu2024enhancing,as2025actsafe}.
Runtime shielding provides a complementary mechanism by filtering,
modifying, or replacing actions before execution
\citep{Mohammed2017,carr2023shielding,yang2023safe}, and recent safe planning
and shielding methods use adaptive conformal prediction to calibrate uncertainty
in safety-critical decision making
\citep{lindemann2023safe_planning,sheng2024pomdp_online_shield,
sheng2024daptive_conformal,scarbo2025conformal}. Control barrier functions
provide certificate-based safety guarantees, but constructing valid barriers
under learned, uncertain, or varying dynamics often requires explicit model
knowledge or robust uncertainty bounds
\citep{choi2020reinforcement,cheng2023safe,ganai2023iterative,
wang2023enforcing,xiao2023bnet}. 
Our work is complementary to these methods: we introduce a
budget-aware Q-Barrier shield for frozen in-context RL policies.

\startpara{In-Context Reinforcement Learning and Safe Adaptation}
In-context reinforcement learning (ICRL) studies agents that adapt to new tasks
at test time without updating parameters, by conditioning the policy on
an expanding interaction history~\citep{moeini2025survey, run2026incontextreinforcementlearningnonstationarity}.
This perspective is
closely related to meta-RL, and contextual or
hidden-parameter RL, where agents infer task or dynamics information from
history and use it to improve behavior online
\citep{finn2017model,yang2019single,chen2021context,chen2022an,
benjamins2023contextualize}. Recent ICRL methods use transformers, recurrent
models, state-space models, or algorithm-distillation objectives to induce
adaptation behavior in the forward pass of a pretrained policy
~\citep{duan2016rl,wang2016learning,xu2022prompting,lee2023supervised,laskin2023incontext,raparthy2023generalization,kirsch2023towards,dai2024context,
dongInContextReinforcementLearning2024,krishnamurthy2024can,grigsby2023amago,grigsby2024amago,
wang2025transformers,wang2025towards}.
Recent theoretical analyses connect TD policy evaluation to specific transformer
constructions, including softmax-attention layers~\citep{xie2026linearattentionsoftmaxtransformers}
and chain-of-thought generation with linear
transformers~\citep{xie2026convergenceemergenceincontextreinforcement}.
Safe meta-RL extends fast adaptation to
constrained settings, but often relies on explicit task-level adaptation,
gradient updates, or additional offline adaptation data
\citep{luo2021mesa,khattar2023a,guan2024cost,xu2025efficient}. Closest to our
setting, SCARED formulates safe ICRL under CMDPs and learns budget-sensitive
policies through Lagrangian pretraining without test-time parameter updates
\citep{moeini2026scared}. Our method focuses on a deployment-time gap: 
with a fixed pretrained ICRL
policy, action selection may still place probability on actions whose predicted
future cost exceeds the remaining budget. We therefore add a runtime barrier
that reweights candidate actions using the current latent context, remaining
budget, and predicted cumulative cost.

\startpara{Latent Representations, Safety Filters, and Auxiliary Training}
Our shield relies on latent representations that support policy adaptation,
latent dynamics prediction, and safety estimation. Prior work on context
encoders, function encoders, and world-model representations shows that learned
latent variables can capture task or dynamics information useful for zero-shot
transfer, in-context adaptation, and hidden-parameter generalization
\citep{luo2022context,tyler2024zero_shot,ingebrand2024zeroshot_ode,
prasanna-rlc24a,koprulu2025safetyprioritizing,kwon2026adaptiveshieldingsafereinforcement}. Closely related to our runtime safety mechanism, recent latent safety-filter methods perform safety analysis
directly in learned world-model spaces using latent reachability,
uncertainty-aware filtering, or runtime-adaptive safety constraints
\citep{Nakamura2025,seo2025uncertainty,agrawal2025anysafe}. These works provide
important precedent for latent-space runtime safety reasoning, but our mechanism
differs: we use a CMDP-style remaining budget, a world-projected latent state,
and a pessimistic cost-to-go critic to define a latent
Q-barrier for candidate-action reweighting. Finally, our training objective uses
auxiliary dynamics and safety losses to shape the representation. Auxiliary
prediction objectives are widely used to improve RL representations
\citep{jaderberg2017reinforcement,laskin2020curl,schwarzer2021pretraining, zhao2023simplified} but joint
policy, dynamics, and safety optimization can create gradient interference. To
reduce this interference, we use an asymmetric stop-gradient scheme inspired by
self-supervised representation learning methods such as BYOL, where
prediction heads are trained against target representations while gradients are
selectively blocked \citep{grill2020bootstrap}. In our setting, the world-model loss
updates the shared encoder, while alignment losses pass through
projection heads with stop-gradient, allowing auxiliary supervision to
regularize the latent space without directly destabilizing policy learning.
 \section{Conclusion}\label{sec:conclusion}
We introduced Q-Barrier, a latent shielding mechanism that makes frozen safe
ICRL policies explicitly budget-aware under test-time generalization shifts,
including both structural OOD and unseen ID settings. The method
augments a safe ICRL policy with shield-supporting modules learned during
training: a world-projected latent representation, a latent dynamics model, and
a cost-critic ensemble. At deployment, all learned parameters remain fixed, and
Q-Barrier only reweights finite action candidates using the remaining budget and
predicted cost-to-go, so adaptation comes from the growing in-context history
rather than online parameter updates.

Our analysis gives a conditional error bound on how the Q-Barrier margin
propagates one step forward over the next candidate set, up to Bellman upper bound
and latent prediction errors. It does not certify closed-loop safety. Empirically, Q-Barrier improves
reward-safety tradeoffs during test-time in-context adaptation. After a short
context window, it achieves higher return in four of five environments while
matching or lowering average episode cost in all five, with one environment
showing a conservative lower-cost/lower-return tradeoff. Across budget sweeps, it
achieves the highest return in four of five environments while generally
matching or reducing average cost relative to baselines.

\textbf{Limitations and future work.}
Q-Barrier depends on latent prediction accuracy, critic calibration, Bellman
residuals, and sampled candidate quality. Overly pessimistic cost estimates can
downweight useful high-return actions. Future work should improve uncertainty
calibration, adapt shield strength to local uncertainty and remaining budget,
extend to multiple constraints, and develop stronger closed-loop analyses for
stochastic shielded policies.

\section*{Acknowledgments}
This work was supported in part by the U.S. National Science Foundation grant CCF-1942836.

\appendix

\section{Proofs}\label{apx:proof}

\subsection{Proof of Theorem~\ref{thm:barrier_margin_propagation}}
Before proving the main theorem, we first prove a value error lemma.
Under Assumptions~\ref{ass:approx_projected_latent_grounding}
and~\ref{ass:q_regular}, the lemma shows that latent prediction error can change
the estimated next-step continuation value by at most
$L_Q\varepsilon_{\mathrm{pred}}$.

\begin{lemma}
\label{lem:value_perturbation_grounding}
Under Assumptions~\ref{ass:approx_projected_latent_grounding} and~\ref{ass:q_regular},
\begin{equation}\label{eq:value_error_bound}
\left|
\hat V_{C,t+1}^{+, k}\!\bigl(f_z(Z_t^{w,k},A_t^k)\bigr)
-
\hat V_{C,t+1}^{+, k}(Z_{t+1}^{w,k})
\right|
\le
L_Q\varepsilon_{\mathrm{pred}}.    
\end{equation}
\end{lemma}

\begin{proof}
Fix episode $k$ and timestep $t$, and condition on the realized next-step
candidate set $\mathcal A_{t+1,k}^{\mathrm{cand}}$. Write $\hat V_C^+(Z)
\doteq
\hat V_{C,t+1}^{+,k}(Z)
=
\min_{A'\in\mathcal A_{t+1,k}^{\mathrm{cand}}}
\hat Q_C^+(Z,A').$
For readability, we omit $+$ and $k$ in intermediate
expressions when there is no ambiguity.
Since $\mathcal A_{t+1,k}^{\mathrm{cand}}$ is fixed in this argument and each
function $\hat Q_C^+(\cdot,A')$ is $L_Q$-Lipschitz, their pointwise minimum is
also $L_Q$-Lipschitz. To see this, let
$A^\star\in\arg\min_{A'\in\mathcal A_{t+1,k}^{\mathrm{cand}}}
\hat Q_C^+(Z',A')$. Then
$\hat V_C^+(Z)-\hat V_C^+(Z')
\le
\hat Q_C^+(Z,A^\star)-\hat Q_C^+(Z',A^\star)
\le
L_Q\|Z-Z'\|_2.$
Swapping $Z$ and $Z'$ leads to $|\hat V_C^+(Z)-\hat V_C^+(Z')|
\le
L_Q\|Z-Z'\|_2.$ Applying this bound with
$Z=f_z(Z_t^{w,k},A_t^k)$ and $Z'=Z_{t+1}^{w,k}$, and using
Assumption~\ref{ass:approx_projected_latent_grounding}, proves
Equation~\eqref{eq:value_error_bound}.
\end{proof}

\begin{proof}
Fix episode $k$ and timestep $t$, and condition on the realized next-step
candidate set $\mathcal A_{t+1,k}^{\mathrm{cand}}$. We follow the same convention used in Proof of Lemma~\ref{lem:value_perturbation_grounding}.

Starting from the next-step state barrier and using the budget recursion
$B_{t+1}^k=B_t^k-C_{t+1}^k$, we have
\begin{align*}
b_{V,t+1}^k(Z_{t+1}^{w,k},B_{t+1}^k)
&=
B_{t+1}^k-\hat V_C^+(Z_{t+1}^{w,k}) \\
&=
B_t^k-C_{t+1}^k-\hat V_C^+(Z_{t+1}^{w,k}) \\
&=
\underbrace{
B_t^k-\hat Q_C^+(Z_t^{w,k},A_t^k)
}_{b_{Q,t}^k(Z_t^{w,k},B_t^k,A_t^k)}
\\
&\quad+
\Bigl(
\hat Q_C^+(Z_t^{w,k},A_t^k)
-
C_{t+1}^k
-
\hat V_C^+(f_z(Z_t^{w,k},A_t^k))
\Bigr)
\\
&\quad+
\Bigl(
\hat V_C^+(f_z(Z_t^{w,k},A_t^k))
-
\hat V_C^+(Z_{t+1}^{w,k})
\Bigr).
\end{align*}
By the definition of the Bellman upper-bound residual,
$$\hat Q_C^+(Z_t^{w,k},A_t^k)
-
C_{t+1}^k
-
\hat V_C^+(f_z(Z_t^{w,k},A_t^k))
\ge
-\Delta_{\mathrm{bell},t}^{+,k}.$$
By Lemma~\ref{lem:value_perturbation_grounding},
$\hat V_C^+(f_z(Z_t^{w,k},A_t^k))
-
\hat V_C^+(Z_{t+1}^{w,k})
\ge
-L_Q\varepsilon_{\mathrm{pred}}.
$
Substituting both inequalities gives
$$b_{V,t+1}^k(Z_{t+1}^{w,k},B_{t+1}^k)
\ge
b_{Q,t}^k(Z_t^{w,k},B_t^k,A_t^k)
-
\Delta_{\mathrm{bell},t}^{+,k}
-
L_Q\varepsilon_{\mathrm{pred}},$$
which proves the claim.
\end{proof}

\subsection{Episode-Level Budget Bound}
\label{apx:episode_bound}
Theorem~\ref{thm:barrier_margin_propagation} is a one-step result.
This subsection shows that it can be extended into an episode-level bound on cumulative budget accumulation by introducing a term that captures the gap between the action actually selected and the critic-minimizing action.

\begin{definition}[Selection Gap]
\label{def:selection_gap}
For the action selected at time $t$ in episode $k$, define
\begin{equation}
\Delta_{\mathrm{sel},t}^k
\doteq
\hat Q_C^+(Z_t^{w,k},A_t^k)
-
\hat V_{C,t}^{+,k}(Z_t^{w,k})
\ge 0.
\label{eq:selection_gap}
\end{equation}
\end{definition}
This quantity is the excess predicted cost-to-go of the selected action over the minimum available in the candidate set at that state.
It is zero when the deployed policy selects the critic-minimizing action, and positive when stochastic shielding picks a higher-cost action.Recall that 
$$b_{V,t}^k(Z,B)
\doteq
B-\hat V_{C,t}^{+,k}(Z),
\quad
b_{Q,t}^k(Z,B,A)
\doteq
B-\hat Q_C^+(Z,A).$$
For continuous-action environments, all quantities are conditioned on the
realized finite candidate sets. The candidate set
$\mathcal A_{t+1,k}^{\mathrm{cand}}$ used in the Bellman residual at time $t$
is the same candidate set used to define
$b_{V,t+1}^k$ at the next decision making.
We assume that each realized candidate set is nonempty and that the executed
action $A_t^k$ is selected from $\mathcal A_{t,k}^{\mathrm{cand}}$.
Thus $\Delta_{\mathrm{sel},t}^k \ge 0$. In continuous-action
domains, the proposition is conditional on the entire realized sequence of
finite candidate sets.

\begin{proposition}[Episode-Level Budget Bound]
\label{prop:episode_budget_bound}
Suppose Theorem~\ref{thm:barrier_margin_propagation} holds along episode $k$
with the signed budget recursion
$B_{t+1}^k=B_t^k-C_{t+1}^k$. Assume the terminal condition $\hat V_{C,T_k}^{+,k}(Z_{T_k}^{w,k})=0.$
Then
\begin{equation}
G_{c,0}(\tau_k)-B_0^k
\le
-b_{V,0}^k(Z_0^{w,k},B_0^k)
+
\sum_{t=0}^{T_k-1}
\Bigl(
\Delta_{\mathrm{sel},t}^k
+
\Delta_{\mathrm{bell},t}^{+,k}
+
L_Q\varepsilon_{\mathrm{pred}}
\Bigr).
\label{eq:episode_budget_bound}
\end{equation}
In particular, if $B_0^k=\delta$ and
$b_{V,0}^k(Z_0^{w,k},\delta)\ge0$, then
\begin{equation}
G_{c,0}(\tau_k)-\delta
\le
\sum_{t=0}^{T_k-1}
\Bigl(
\Delta_{\mathrm{sel},t}^k
+
\Delta_{\mathrm{bell},t}^{+,k}
+
L_Q\varepsilon_{\mathrm{pred}}
\Bigr).
\label{eq:episode_budget_bound_feasible}
\end{equation}
Averaging over $K$ episodes and taking expectation gives
\begin{equation}
\frac{1}{K}\sum_{k=1}^K
\mathbb{E}\bigl[G_{c,0}(\tau_k)\bigr]
\le
\delta
+
\frac{1}{K}\sum_{k=1}^K
\mathbb{E}\!\left[
\sum_{t=0}^{T_k-1}
\Bigl(
\Delta_{\mathrm{sel},t}^k
+
\Delta_{\mathrm{bell},t}^{+,k}
+
L_Q\varepsilon_{\mathrm{pred}}
\Bigr)
-
b_{V,0}^k(Z_0^{w,k},\delta)
\right].
\label{eq:aggregate_budget_bound}
\end{equation}
\end{proposition}

\begin{proof}
By the definition of the selection gap,
\[
\Delta_{\mathrm{sel},t}^k
=
\hat Q_C^+(Z_t^{w,k},A_t^k)
-
\hat V_{C,t}^{+,k}(Z_t^{w,k}).
\]
Therefore,
\begin{align*}
b_{Q,t}^k(Z_t^{w,k},B_t^k,A_t^k)
&=
B_t^k-\hat Q_C^+(Z_t^{w,k},A_t^k) \\
&=
B_t^k-\hat V_{C,t}^{+,k}(Z_t^{w,k})
-
\Delta_{\mathrm{sel},t}^k \\
&=
b_{V,t}^k(Z_t^{w,k},B_t^k)
-
\Delta_{\mathrm{sel},t}^k .
\end{align*}
Substituting this identity into
Theorem~\ref{thm:barrier_margin_propagation} gives
\begin{equation}
b_{V,t+1}^k(Z_{t+1}^{w,k},B_{t+1}^k)
\ge
b_{V,t}^k(Z_t^{w,k},B_t^k)
-
\Delta_{\mathrm{sel},t}^k
-
\Delta_{\mathrm{bell},t}^{+,k}
-
L_Q\varepsilon_{\mathrm{pred}} .
\label{eq:recursive_margin}
\end{equation}
Summing Equation~\eqref{eq:recursive_margin} over
$t=0,\dots,T_k-1$ and telescoping provides
\[
b_{V,T_k}^k(Z_{T_k}^{w,k},B_{T_k}^k)
\ge
b_{V,0}^k(Z_0^{w,k},B_0^k)
-
\sum_{t=0}^{T_k-1}
\Bigl(
\Delta_{\mathrm{sel},t}^k
+
\Delta_{\mathrm{bell},t}^{+,k}
+
L_Q\varepsilon_{\mathrm{pred}}
\Bigr).
\]
By the terminal condition
$\hat V_{C,T_k}^{+,k}(Z_{T_k}^{w,k})=0$, we have
\[
b_{V,T_k}^k(Z_{T_k}^{w,k},B_{T_k}^k)
=
B_{T_k}^k.
\]
Using the budget identity
\[
B_{T_k}^k
=
B_0^k-G_{c,0}(\tau_k),
\]
we obtain
\[
B_0^k-G_{c,0}(\tau_k)
\ge
b_{V,0}^k(Z_0^{w,k},B_0^k)
-
\sum_{t=0}^{T_k-1}
\Bigl(
\Delta_{\mathrm{sel},t}^k
+
\Delta_{\mathrm{bell},t}^{+,k}
+
L_Q\varepsilon_{\mathrm{pred}}
\Bigr).
\]
Finally, rearranging gives Equation~\eqref{eq:episode_budget_bound}.
Setting $B_0^k=\delta$ and using
$b_{V,0}^k(Z_0^{w,k},\delta)\ge0$ gives
Equation~\eqref{eq:episode_budget_bound_feasible}. Averaging
Equation~\eqref{eq:episode_budget_bound} over $k$ and taking expectations gives
Equation~\eqref{eq:aggregate_budget_bound}.
\end{proof}

\paragraph{Remark.}
Proposition~\ref{prop:episode_budget_bound} connects the local one-step barrier analysis to the aggregate budget objective through the selection gap $\Delta_{\mathrm{sel},t}^k$, which measures how much the chosen action's predicted cost exceeds the lowest predicted cost among candidates. For the soft shield, this term captures the cost of preserving policy support rather than always selecting the critic-greedy lowest-cost action.
It decomposes the episode-level budget into residuals from action selection, critic conservatism, and latent prediction error.

\subsection{Overlap of Consecutive Safe Sets}
\label{apx:proof_lemma_overlap}

The following result gives a local interpretation of the Lipschitz regularity assumption. Since candidate
sets may be resampled over time, the statement is conditional on an action being present in both
consecutive candidate sets. 

\begin{definition}[Safe Candidate Set]
\label{def:critic_safe_set_appendix}
For the realized candidate set $\mathcal A_{t,k}^{\mathrm{cand}}$, world-projected latent state
$Z\in\mathcal Z^w$, and remaining budget $B$, define
\begin{equation}
\mathcal A_{\mathrm{safe},t}^k(Z,B)
\doteq
\left\{
A\in \mathcal A_{t,k}^{\mathrm{cand}}:
\hat Q^+_C(Z,A)\le B
\right\}.
\label{eq:critic_safe_set_appendix}
\end{equation}
\end{definition}

\begin{lemma}[Conditional Overlap of Consecutive Safe Candidate Sets]
\label{lem:safe_overlap_appendix}
Suppose Assumption~\ref{ass:q_regular} holds.
Let $A\in
\mathcal A_{t,k}^{\mathrm{cand}}
\cap
\mathcal A_{t+1,k}^{\mathrm{cand}}$
satisfy
\begin{equation}
\hat Q^+_C(Z_t^{w,k},A)
\le
B_t^k-\eta
\label{eq:slackened_constraint_appendix}
\end{equation}
for some margin $\eta>0$.
If
\begin{equation}
L_Q\|Z_{t+1}^{w,k}-Z_t^{w,k}\|_2
+
|B_{t+1}^k-B_t^k|
\le
\eta,
\label{eq:overlap_condition_appendix}
\end{equation}
then
\begin{equation}
A\in
\mathcal A_{\mathrm{safe},t+1}^k(Z_{t+1}^{w,k},B_{t+1}^k).
\end{equation}
Consequently, if such an action exists, then
\begin{equation}
\mathcal A_{\mathrm{safe},t}^k(Z_t^{w,k},B_t^k)
\cap
\mathcal A_{\mathrm{safe},t+1}^k(Z_{t+1}^{w,k},B_{t+1}^k)
\neq
\varnothing .
\end{equation}
\end{lemma}

\begin{proof}
Because $A\in\mathcal A_{t,k}^{\mathrm{cand}}\cap\mathcal A_{t+1,k}^{\mathrm{cand}}$, it is eligible at both consecutive decision times.
By Assumption~\ref{ass:q_regular},
$$|\hat Q^+_C(Z_{t+1}^{w,k},A)-\hat Q^+_C(Z_t^{w,k},A)|
\le
L_Q\|Z_{t+1}^{w,k}-Z_t^{w,k}\|_2.$$
Using Equation \eqref{eq:slackened_constraint_appendix}, we have
$$\hat Q^+_C(Z_{t+1}^{w,k},A)
\le
B_t^k-\eta
+
L_Q\|Z_{t+1}^{w,k}-Z_t^{w,k}\|_2 .$$
Therefore,
$$\hat Q^+_C(Z_{t+1}^{w,k},A)-B_{t+1}^k
\le
L_Q\|Z_{t+1}^{w,k}-Z_t^{w,k}\|_2
-\eta
+
|B_{t+1}^k-B_t^k|.$$
By Equation \eqref{eq:overlap_condition_appendix}, the right-hand side is less or equal than $0$, so
$$\hat Q^+_C(Z_{t+1}^{w,k},A)\le B_{t+1}^k.$$
Since $A\in\mathcal A_{t+1,k}^{\mathrm{cand}}$, this implies
$A\in \mathcal A_{\mathrm{safe},t+1}^k(Z_{t+1}^{w,k},B_{t+1}^k).$
The overlap claim follows because Equation \eqref{eq:slackened_constraint_appendix} also implies
$A\in\mathcal A_{\mathrm{safe},t}^k(Z_t^{w,k},B_t^k)$.
\end{proof}
This overlap result is mainly relevant for discrete action spaces or continuous
implementations that reuse candidate actions across consecutive steps. With
independent continuous resampling, the exact overlap condition is less likely to hold.

\subsection{Pessimistic Ensemble Cost Critics}
\label{apx:bellman_details}

This section describes how the ensemble cost critics are trained and how the
pessimistic ensemble estimate used by the Q-barrier relates to the Bellman
upper-bound error $\Delta_{\mathrm{bell},t}^{+, k}$ in
Equation~\eqref{eq:bellman_error}.

\paragraph{Critic training.}
The discounted, episode-masked targets and PopArt-normalized critic loss are
defined in Equation~\eqref{eq:training_critic_losses}. We use $M=4$ cost critics,
enumerate discrete target-policy actions, and use one action sample for
continuous actions. The target takes the maximum over two target critics
sampled with replacement. This training target is distinct from the
undiscounted candidate-minimum residual in Equation~\eqref{eq:bellman_error}.

\paragraph{Ensemble pessimism.}
The Q-barrier uses the ensemble maximum in the budget's cost units,
$\hat Q_C^+(Z_t^{w,k},A)
\doteq
\max_{i\in\{1,\ldots,M\}}
\hat Q_{C,i}(Z_t^{w,k},A).$
The barrier margin is then computed as $b^k_{Q,t}(Z_t^{w,k},B_t^k,A)
=
B_t^k-\hat Q_C^+(Z_t^{w,k},A).$
This worst-case aggregation makes the shield conservative when critic heads
disagree, reducing the chance that an action is treated as safe because of a
single underestimated cost prediction.

\textbf{Effect on $\Delta_{\mathrm{bell},t}^k$.}
The pessimistic ensemble is a deployment-time heuristic intended to reduce $\Delta_{\mathrm{bell},t}^k$ empirically by biasing cost estimates conservatively in states where critic disagreement is large. The role of the ensemble maximum is to make underestimation less likely in uncertain states. Related empirical behavior is examined through the discounted critic diagnostics in Appendix~\ref{app:diagnostics}.

\clearpage
\section{Empirical Diagnostics for the Barrier-Margin Analysis}
\label{app:diagnostics}
\begin{table*}[t]
\centering
\resizebox{\textwidth}{!}{\begin{tabular}{lccccc}
\toprule
\textbf{Environment}
& $e_{\mathrm{pred}}$ $\downarrow$
& $e_V^+$ $\downarrow$
& $\widehat L_Q^{\mathrm{local},+}$ $\downarrow$
& Bellman sat. $\hat Q_C^+$ $\uparrow$
& $\Delta_{\mathrm{diag}}^+$ $\downarrow$ \\
\midrule
SafeDarkRoom
& $0.00252 \pm 0.00002$
& $0.3595 \pm 0.0056$
& $228.8 \pm 3.0$
& $72.2 \pm 1.0$
& $67.19 \pm 3.20$ \\
SafeDarkMujoco-Point
& $0.00766 \pm 0.00056$
& $0.0847 \pm 0.0053$
& $14.10 \pm 0.22$
& $89.8 \pm 0.9$
& $0.248 \pm 0.048$ \\
SafeDarkMujoco-Car
& $0.00906 \pm 0.00058$
& $0.0972 \pm 0.0061$
& $13.08 \pm 0.14$
& $89.8 \pm 0.4$
& $0.104 \pm 0.004$ \\
SafeVelocity-HalfCheetah
& $0.0460 \pm 0.0008$
& $0.0432 \pm 0.0009$
& $0.872 \pm 0.003$
& $70.5 \pm 0.2$
& $0.386 \pm 0.005$ \\
SafeVelocity-Ant
& $0.00478 \pm 0.00013$
& $0.0514 \pm 0.0019$
& $10.35 \pm 0.14$
& $89.9 \pm 0.5$
& $0.719 \pm 0.032$ \\
\bottomrule
\end{tabular}}
\caption{
\textbf{Model and critic diagnostics on OOD trajectories.}
We report latent prediction and value-perturbation diagnostics, together with
discounted Bellman diagnostics. Bellman satisfaction is the percentage of transitions with
$\Delta_{\mathrm{diag},t}\le 10^{-6}$. $\hat Q_C^+$ denotes the pessimistic critic used by the deployed Q-barrier.
Residual magnitudes are reported in
critic-scaled cost units. SafetyHalfCheetah has the largest prediction error,
which is consistent with its more conservative reward-cost behavior under
Q-Barrier shielding.
}
\label{tab:theory_diagnostics_core}
\end{table*}
    
\begin{table*}[t]
\centering
\resizebox{\textwidth}{!}{\begin{tabular}{lcccc}
\toprule
\textbf{Environment}
& Bellman sat. mean $\uparrow$
& Bellman sat. $\hat Q_C^+$ $\uparrow$
& $\Delta_{\mathrm{diag}}^{\mathrm{mean}}$ $\downarrow$
& $\Delta_{\mathrm{diag}}^+$ $\downarrow$ \\
\midrule
SafeDarkRoom
& $66.7 \pm 1.1$
& $72.2 \pm 1.0$
& $169.0 \pm 8.2$
& $67.19 \pm 3.20$ \\
SafeDarkMujoco-Point
& $93.2 \pm 0.6$
& $89.8 \pm 0.9$
& $0.053 \pm 0.007$
& $0.248 \pm 0.048$ \\
SafeDarkMujoco-Car
& $78.6 \pm 0.5$
& $89.8 \pm 0.4$
& $11.0 \pm 0.51$
& $0.104 \pm 0.004$ \\
SafeVelocity-HalfCheetah
& $72.8 \pm 0.2$
& $70.5 \pm 0.2$
& $0.257 \pm 0.004$
& $0.386 \pm 0.005$ \\
SafeVelocity-Ant
& $88.6 \pm 0.5$
& $89.9 \pm 0.5$
& $1.089 \pm 0.038$
& $0.719 \pm 0.032$ \\
\bottomrule
\end{tabular}}
\caption{
\textbf{Effect of pessimistic critic aggregation on Bellman diagnostics.}
We use the same layout in Table~\ref{tab:theory_diagnostics_core}.
}
\label{tab:theory_diagnostics_bellman}
\end{table*}

The barrier-margin bound in Theorem~\ref{thm:barrier_margin_propagation}
depends on learned quantities that are not guaranteed by construction: one-step
latent prediction error, local sensitivity of the learned cost critic, and the
Bellman upper-bound residual of the critic. We therefore report related
empirical diagnostics on evaluation trajectories using the deployed soft
Q-barrier shield with $N_s=8$ candidate actions. In the main text, $\hat Q_C^+$ denotes the pessimistic ensemble critic used by
the deployed barrier. In this appendix, we also report diagnostics for the
ensemble mean critic $\hat Q_C^{\mathrm{mean}}$ to isolate the effect of
pessimistic aggregation. In this subsection, critic values and costs use the
same training scale $s_C$, and we omit the scale superscript and $k$ for simplicity.

For each selected transition $(Z_t^w,A_t,C_{t+1},Z_{t+1}^w)$, we compute the
one-step latent prediction error
\begin{equation}
e_{\mathrm{pred},t}
=
\left\|
f_z(Z_t^w,A_t)-Z_{t+1}^w
\right\|_2 .
\label{eq:diag_pred_error}
\end{equation}
To evaluate the value perturbation induced by this latent error, we use the
same next-step candidate set $\mathcal A_{t+1}^{\mathrm{cand}}$ for both the
predicted and realized next latent states:
\begin{equation}
\hat V_{C,t+1}^{+}(Z)
=
\min_{A'\in\mathcal A_{t+1}^{\mathrm{cand}}}
\hat Q_C^{+}(Z,A').
\label{eq:diag_value_def}
\end{equation}
The directly induced value perturbation is
\begin{equation}
e_{V,t}^{+}
=
\left|
\hat V_{C,t+1}^{+}\!\left(f_z(Z_t^w,A_t)\right)
-
\hat V_{C,t+1}^{+}\!\left(Z_{t+1}^w\right)
\right|.
\label{eq:diag_value_perturb}
\end{equation}
This is the most direct empirical analogue of the value perturbation term in
Lemma~\ref{lem:value_perturbation_grounding}.

We also report a finite-difference proxy for local critic sensitivity:
\begin{equation}
\widehat L_{Q,t}^{\mathrm{local},+}
=
\max_{A'\in\mathcal A_{t+1}^{\mathrm{cand}}}
\frac{
\left|
\hat Q_C^{+}(f_z(Z_t^w,A_t),A')
-
\hat Q_C^{+}(Z_{t+1}^w,A')
\right|
}
{
\left\|
f_z(Z_t^w,A_t)-Z_{t+1}^w
\right\|_2+\epsilon
},
\qquad
\epsilon=10^{-8}.
\label{eq:diag_local_lipschitz}
\end{equation}
This quantity should be interpreted only as a local finite-difference diagnostic:
it can become large when the latent prediction error in the denominator is very
small. For this reason, we use $e_{V,t}^{+}$ as the primary diagnostic for the
critic-scaled analogue of the value perturbation in the theorem.

Finally, we compute the discounted Bellman diagnostic:
\begin{equation}
\Delta_{\mathrm{diag},t}^{+}
=
\left[
C_{t+1}
+
\gamma_*\chi_{t+1}\hat V_{C,t+1}^{+}\!\left(f_z(Z_t^w,A_t)\right)
-
\hat Q_C^{+}(Z_t^w,A_t)
\right]_+ .
\label{eq:diag_bellman}
\end{equation}
Here $\chi_{t+1}$ is the evaluation bootstrap mask. Unlike
Equation~\eqref{eq:bellman_error}, this diagnostic includes discounting and
critic cost scaling, so it is not a direct measurement of that undiscounted residual.
A transition satisfies this discounted upper-bound check when
$\Delta_{\mathrm{diag},t}^{+}\le10^{-6}$. We also compute the same residual for
the ensemble mean critic $\hat Q_C^{\mathrm{mean}}$ to isolate the
effect of pessimistic aggregation.

All diagnostics are computed after action selection and are not used by the
policy or shield during deployment. We first aggregate metrics within each
evaluation rollout, using $50$ tasks, and then report mean $\pm$ standard errors across rollouts.

\begin{figure*}[t]
    \centering
    \includegraphics[width=1.0\linewidth]{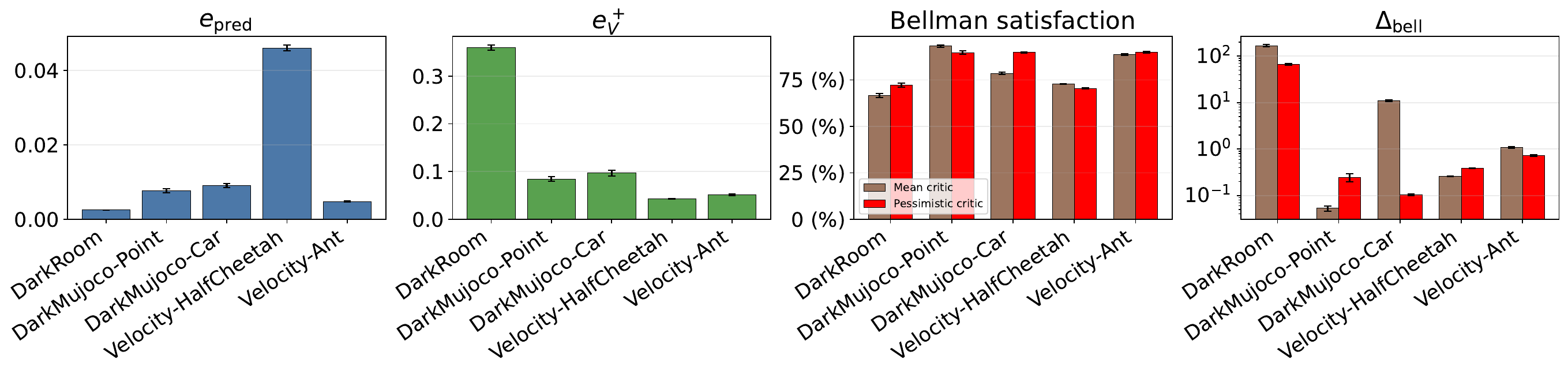}
    \caption{
    \textbf{Model and critic diagnostics on OOD evaluation trajectories.}
    We report one-step latent prediction error, induced value perturbation,
    Bellman upper-bound satisfaction, and positive Bellman residual. Error bars
    show standard error over rollout-level summaries. We compare the
    ensemble mean critic with the pessimistic critic used by the deployed
    Q-barrier. Bellman residuals are plotted on a log scale because their
    magnitudes differ across critic scales and environments.
    }
    \label{fig:theory_diagnostics}
\end{figure*}

Table~\ref{tab:theory_diagnostics_core} shows that the theorem-relevant
quantities vary across environments. SafetyHalfCheetah has
the largest latent prediction error, which is consistent with the more
conservative behavior observed in the main experiments. SafeDarkRoom has a small
latent prediction error but a large local finite-difference sensitivity and a
large Bellman residual in critic-scaled units, indicating that its learned critic
changes sharply even under small latent perturbations.

Table~\ref{tab:theory_diagnostics_bellman} isolates the effect of pessimistic
critic aggregation on the Bellman upper-bound residual. The pessimistic critic
improves Bellman satisfaction or reduces residual magnitude in SafeDarkRoom,
SafetyCar, and SafetyAnt. The effect is not uniform: in
SafetyPoint and SafetyHalfCheetah, pessimistic aggregation
slightly lowers Bellman satisfaction and increases the positive residual. Max aggregation raises predicted cost, but it
does not enforce the Bellman upper bound on every transition.
\section{Hard Shield Variant}
\label{app:hard_shield}

The main empirical results use the soft rule in
Algorithm~\ref{alg:runtime_shielding}, which preserves support over the finite
candidate set while downweighting actions with negative Q-barrier margin. Here,
we define its hard truncation variant. The hard shield restricts
action selection to the barrier-feasible subset of the candidate set whenever
that subset is nonempty.

For the realized finite candidate set $\mathcal A_{t,k}^{\mathrm{cand}}$, define
the hard safe set
\begin{equation}
\mathcal A_{\mathrm{safe},t}^k(Z_t^{w,k},B_t^k)
\doteq
\left\{
A\in\mathcal A_{t,k}^{\mathrm{cand}}:
b^k_{Q,t}(Z_t^{w,k},B_t^k,A)\ge 0
\right\}
=
\left\{
A\in\mathcal A_{t,k}^{\mathrm{cand}}:
\hat Q_C^+(Z_t^{w,k},A)\le B_t^k
\right\}.
\label{eq:critic_safe_set_appendix}
\end{equation}
As in the soft shield, define the finite-candidate base weight
\[
\rho_t(A)
=
\begin{cases}
\pi_\theta(A\mid Z_t^{p,k}),
& \text{if the discrete action space is enumerated},\\
1,
& \text{if } A\in\mathcal A_{t,k}^{\mathrm{cand}}
\text{ is sampled from } \pi_\theta(\cdot\mid Z_t^{p,k}).
\end{cases}
\]
The hard-shield policy over the finite candidate set is
\begin{equation}
\pi_{\mathrm{hard}}(A \mid S_t^k,Z_t^k,B_t^k)
=
\begin{cases}
\dfrac{
\rho_t(A)\,
\mathbb{I}\!\left[
A\in\mathcal A_{\mathrm{safe},t}^k(Z_t^{w,k},B_t^k)
\right]
}{
\sum_{A'\in\mathcal A_{t,k}^{\mathrm{cand}}}
\rho_t(A')\,
\mathbb{I}\!\left[
A'\in\mathcal A_{\mathrm{safe},t}^k(Z_t^{w,k},B_t^k)
\right]
},
& \text{if the denominator$>0$}, \\[2.2ex]
\dfrac{
\mathbb{I}\!\left[
A\in\mathcal A_{\min,t}^k
\right]
}{
|\mathcal A_{\min,t}^k|
},
& \text{otherwise},
\end{cases}
\label{eq:hard_shield}
\end{equation}
where
\begin{equation}
\mathcal A_{\min,t}^k
\doteq
\arg\min_{A'\in\mathcal A_{t,k}^{\mathrm{cand}}}
\hat Q_C^+(Z_t^{w,k},A').
\label{eq:hard_shield_fallback}
\end{equation}

If the safe candidate set is nonempty, the hard shield renormalizes the
finite-candidate base weights over barrier-feasible candidates only. If the safe
set is empty, it falls back to the candidate with the smallest pessimistic
predicted cost, with uniform tie-breaking if multiple candidates attain the
minimum.

In discrete-action settings, $\rho_t(A)$ is the base-policy probability mass, so
Equation~\eqref{eq:hard_shield} is a standard truncation and renormalization of
the backbone policy. In continuous-action settings, the candidate actions are
already sampled from the base policy, so $\rho_t(A)=1$ and the hard shield is a
finite-sample rejection/resampling rule over the sampled candidates. It is
therefore a hard filter on the sampled candidate set, not a
truncation of the full continuous policy density.

\clearpage
\section{Ablation: Hard Q-Barrier vs.\ Soft Q-Barrier}
\label{app:ablation_shield_type}
\begin{figure*}[t]
    \centering
    \begin{subfigure}[t]{0.8\linewidth}
        \centering
        \includegraphics[width=1.0\linewidth]{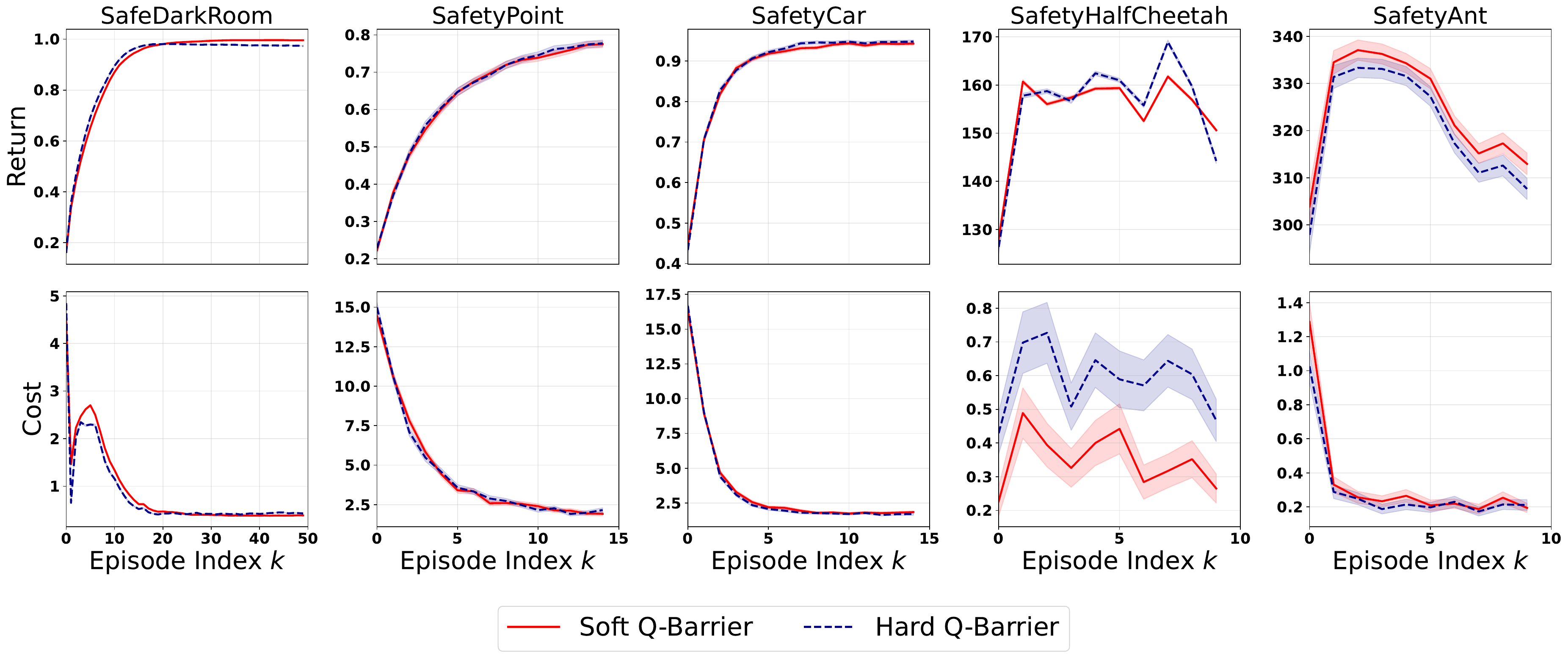}
        \caption{In-context adaptation dynamics
    }
        \label{fig:ablation_sampling_soft}
    \end{subfigure}

    \vspace{0.5em}

    \begin{subfigure}[t]{0.8\linewidth}
        \centering
        \includegraphics[width=1.0\linewidth]{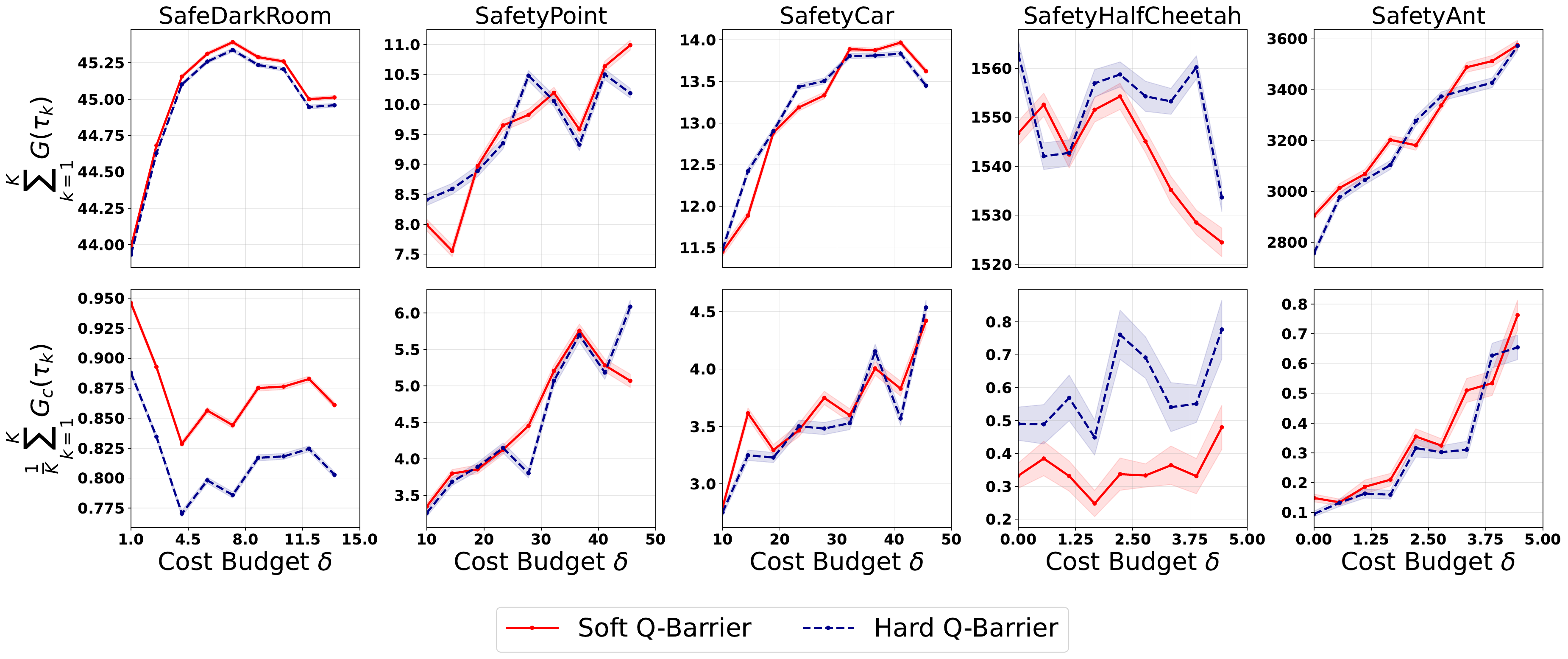}
        \caption{Reward-cost tradeoffs under budget variation}
        \label{fig:ablation_sampling_hard}
    \end{subfigure}

        \caption{
        \textbf{Soft versus hard Q-barrier shielding.}
    We compare the soft shield used in the main experiments with a hard
    truncation variant that selects only from barrier-feasible candidates when
    possible. The two variants have similar reward--safety profiles in most
    environments. Hard shielding gives marginal cost reductions in several
    environments, such as SafeDarkRoom, but these gains can come with small
    return loss and are not uniform. In SafetyHalfCheetah, hard shielding
    achieves the highest return but also increases cost, consistent with the
    larger latent prediction error reported in
    Table~\ref{tab:theory_diagnostics_core}. This suggests that hard truncation
    is not automatically safer when the learned barrier is affected by model or
    critic error.
    }
    \label{fig:ablation_hard_soft}
\end{figure*}
We evaluate both shield types on the OOD tasks from the main experiments. The
only difference between the two variants is the deployment-time action-selection
rule: soft reweighting versus hard truncation. All other components are shared.

Figure~\ref{fig:ablation_hard_soft} shows that hard and soft shielding perform
similarly overall. In SafeDarkRoom, hard shielding slightly reduces cost relative
to soft shielding, but with a small return drop, suggesting that the soft shield
already assigns most probability to low-risk candidates.

SafetyHalfCheetah is the main exception: hard shielding achieves higher return
than soft shielding, but also increases cost. This is consistent with the
diagnostics in Table~\ref{tab:theory_diagnostics_core}, where SafetyHalfCheetah
has the largest latent prediction error. With less accurate latent prediction,
hard thresholding can be less stable: small barrier errors may flip candidates
between safe and unsafe, and fallback or truncation can amplify these mistakes.
Soft shielding is smoother because it changes candidate probabilities
continuously with the predicted barrier margin.

\clearpage
\section{Number of Candidate Actions in Continuous Control}
\label{app:ablation_sampling}
The Q-barrier shield evaluates a finite candidate set before selecting an action. In discrete-action environments this set is the full action space, but in continuous-action environments the shield approximates the budget-aware action distribution using $N_s$ actions sampled from the base policy. This ablation studies how the number of sampled candidates affects the reward-safety tradeoff.

We evaluate $N_s \in \{4,8,16,32\}$ on the continuous-control environments including SafeDarkMujoco and SafeVelocity domains. 
For each $N_s$, we compare the soft shield used in the main experiments with a hard shield that rejects candidate actions whose learned Q-barrier is negative whenever possible. 
\begin{figure*}[t]
    \centering
    \begin{subfigure}[t]{0.8\linewidth}
        \centering
        \includegraphics[width=1.0\linewidth]{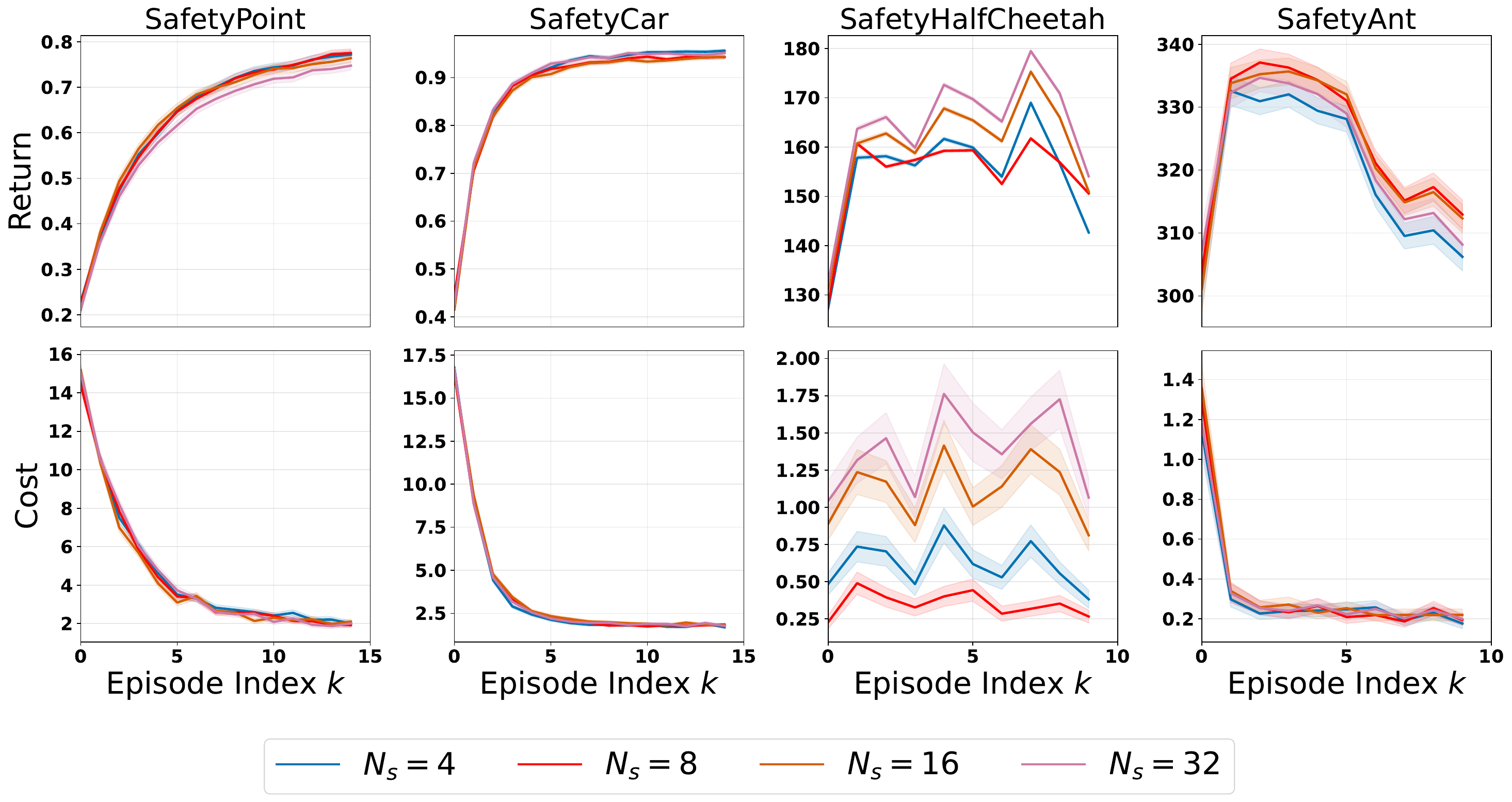}
        \caption{Soft Q-Barrier}
        \label{fig:ablation_sampling_soft}
    \end{subfigure}

    \vspace{0.5em}

    \begin{subfigure}[t]{0.8\linewidth}
        \centering
        \includegraphics[width=1.0\linewidth]{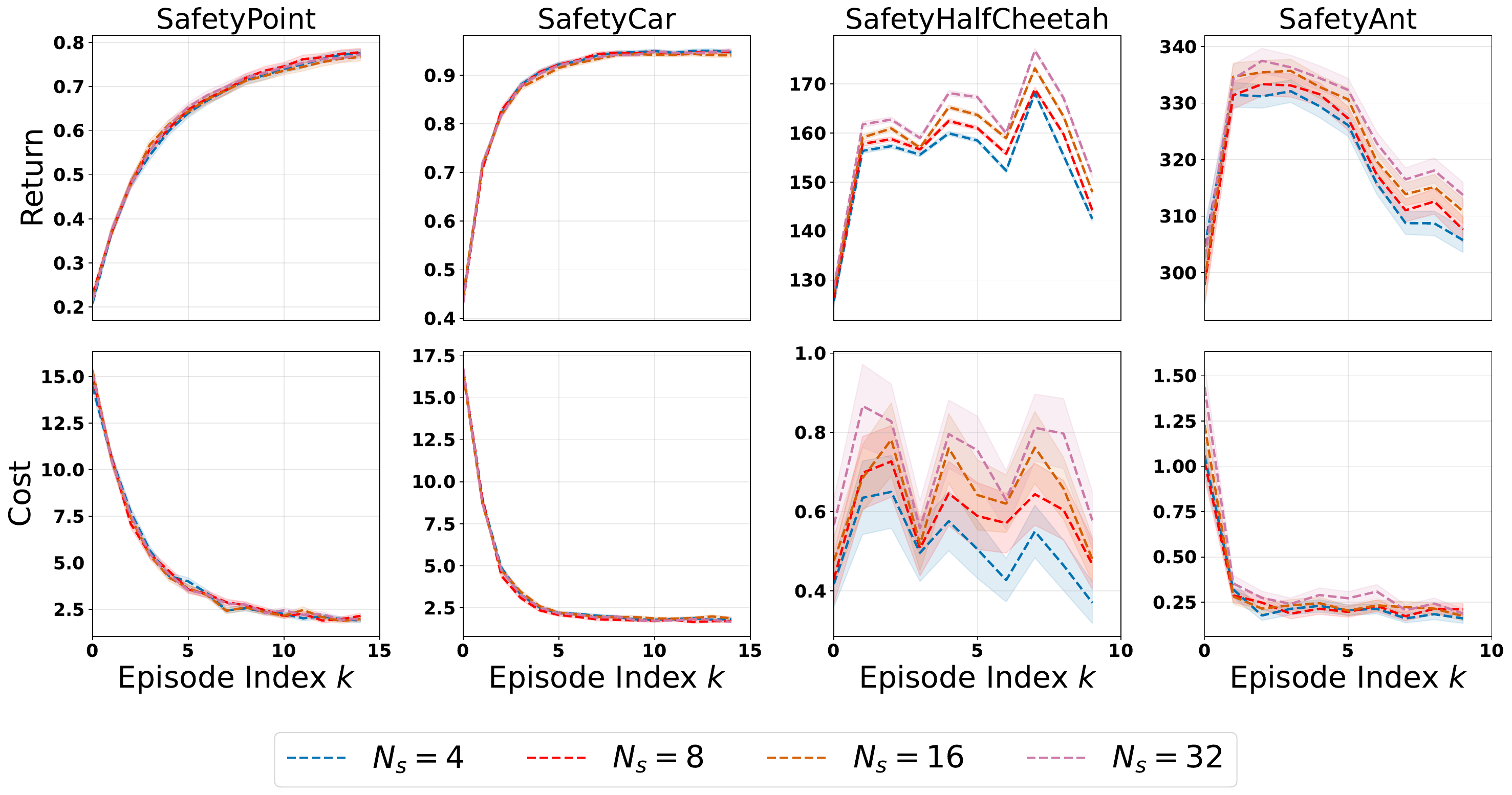}
        \caption{Hard Q-Barrier}
        \label{fig:ablation_sampling_hard}
    \end{subfigure}

    \caption{
    \textbf{Effect of candidate-set size on OOD adaptation dynamics.}
    We vary the number of sampled candidate actions
    $N_s \in \{4,8,16,32\}$ used by our Q-barrier shield during deployment.
    \textbf{(a)} soft filtering, which reweights candidates by their
    barrier values, and \textbf{(b)} hard filtering, which removes candidates that
    fail the barrier test. Curves show per-episode return and cost over the
    in-context OOD evaluation horizon, with shaded regions denoting standard
    errors across rollouts. 
    }
    \label{fig:ablation_sampling}
\end{figure*}

Figure~\ref{fig:ablation_sampling} studies the effect of the number of sampled
candidate actions $N_s$ under soft and hard Q-barrier shielding. The sensitivity
is environment dependent. In SafeDarkMujoco, both return and cost are 
stable across $N_s$ for both filtering rules. Hard filtering (Figure~\ref{fig:ablation_sampling_hard}) is especially
insensitive because it removes candidates that fail the learned barrier test, so
larger candidate sets mainly affect which feasible action is selected. 
Soft filtering (Figure~\ref{fig:ablation_sampling_soft}) instead reweights all sampled candidates according to their
barrier values, making the selected action more susceptible to additional
high-return but higher-cost samples.

The SafeVelocity environments show a stronger sampling effect. In
SafetyHalfCheetah, increasing $N_s$ improves return during adaptation, but also
increases cost. 
This reward-cost tradeoff is more ordered under hard filtering:
larger candidate sets consistently increase return, at the expense of higher
cost, in both SafetyHalfCheetah and SafetyAnt. 
Under soft filtering, the ordering
is less consistent because unsafe or near-unsafe candidates are downweighted
rather than removed, leaving more dependence on the learned barrier scale. 
The stronger sensitivity in locomotion is expected, since these
tasks have higher-dimensional action spaces and a finite candidate set more
directly limits the shield's ability to find high-return safe actions. Overall, the ablation supports using $N_s=8$ as a conservative default in the
main experiments. We do not tune $N_s$ per environment or report the best
sampling number.

\section{Pseudocode}
\label{app:algorithms}
We provide pseudocode for the training procedure and the deployment-time action-selection rule used by the proposed shielding framework.
Algorithm~\ref{alg:train_shielding} summarizes the principal actor--critic and representation-learning objectives.
Its critic notation follows the scaled-head convention in Appendix~\ref{app:loss_details}.
Algorithm~\ref{alg:runtime_shielding} summarizes the runtime Q-barrier shielding rule used at deployment.
Figure~\ref{fig:architecture_core_losses} illustrates the architecture and core
loss routing.
\begin{figure}[htbp]
\centering
\begin{tikzpicture}[
    x=1cm,y=1cm,font=\small,
    coreblock/.style={draw,rounded corners=2pt,align=center,
        minimum height=0.9cm,inner sep=4pt},
    coreflow/.style={->,line width=0.7pt},
    auxflow/.style={->,dashed,draw=black!65,line width=0.7pt}
]
\node[align=center] (history) at (0,0) {Context\\$(H_t,S_t)$};
\node[coreblock,fill=black!5,minimum width=1.6cm] (encoder) at (2.0,0)
    {Shared\\encoder $E_\phi$\\$Z_t$};
\node[coreblock,fill=blue!5,minimum width=1.6cm] (policyproj) at (4.1,1.1)
    {Policy proj.\\$g_\psi^{\mathrm{policy}}$\\$Z_t^p$};
\node[coreblock,fill=orange!8,minimum width=1.6cm] (worldproj) at (4.1,-1.1)
    {World proj.\\$g_\omega^{\mathrm{world}}$\\$Z_t^w$};
\node[coreblock,fill=blue!5,minimum width=1.6cm] (policy) at (6.15,1.1)
    {Policy\\$\pi_\theta$};
\node[coreblock,fill=orange!8,minimum width=1.8cm] (costq) at (6.15,-1.1)
    {Cost critics\\$\{\hat Q_{C,i}\}$};
\node[coreblock,minimum width=1.5cm] (candidates) at (8.2,1.1)
    {Candidates\\$\mathcal A_t^{\mathrm{cand}}$};
\node[coreblock,fill=black!5,minimum width=1.7cm] (shield) at (10.4,0)
    {Q-Barrier\\soft/hard rule};
\node[align=center] (action) at (12.4,0) {Executed\\$A_t$};
\node[coreblock,minimum height=0.7cm] (budget) at (10.4,2.1)
    {Remaining budget\\$B_t$};
\node[coreblock,dashed,draw=black!65,fill=black!3] (rewardq) at (2.0,-3.0)
    {Reward critics\\$\{\hat Q_{R,i}\}$};
\node[coreblock,dashed,draw=black!65,fill=black!3] (worldheads) at (6.15,-3.0)
    {World-model heads\\$p_z,\ \hat R,\ \hat C_{1:2}$};

\draw[coreflow] (history) -- (encoder);
\draw[coreflow] (encoder.east) -- (policyproj.west);
\draw[coreflow] (encoder.east) -- (worldproj.west);
\draw[coreflow] (policyproj) -- (policy);
\draw[coreflow] (policy) -- (candidates);
\draw[coreflow] (worldproj) -- (costq);
\draw[coreflow] (candidates.south) -- node[above,sloped,font=\footnotesize] {$a$} (costq.north);
\draw[coreflow] (candidates.east) -- (shield.north west);
\draw[coreflow] (costq.east) -- node[below,sloped,font=\footnotesize] {pessimistic cost} (shield.south west);
\draw[coreflow] (budget.south) -- (shield.north);
\draw[coreflow] (shield) -- (action);
\draw[auxflow] (encoder.south) -- (rewardq.north);
\draw[auxflow] (worldproj.south) |- (worldheads.west);
\draw[auxflow] (worldproj.north) -- (policyproj.south);
\node[align=left,font=\footnotesize,anchor=west] at (4.25,0)
    {$\operatorname{sg}$ targets\\alignment losses};
\draw[coreflow] (8.6,-2.6) -- (9.3,-2.6);
\node[anchor=west,font=\footnotesize] at (9.4,-2.6) {deployment path};
\draw[auxflow] (8.6,-3.15) -- (9.3,-3.15);
\node[anchor=west,font=\footnotesize] at (9.4,-3.15) {training-only branch};
\end{tikzpicture}

\medskip
{\small\setlength{\tabcolsep}{5pt}
\begin{tabular}{lccccccc}
\toprule
Core loss & $E_\phi$ & $g_\psi^{\mathrm{policy}}$ & $g_\omega^{\mathrm{world}}$
& $\pi_\theta$ & $\hat Q_R$ & $\hat Q_C$ & $p_z,\hat R,\hat C$ \\
\midrule
Actor & $+$ & $+$ & -- & $+$ & -- & -- & -- \\
Critic TD & $+$ & -- & $+$ & -- & $+$ & $+$ & -- \\
World-model & $+$ & -- & $+$ & -- & -- & -- & $+$ \\
Distillation / conjugacy & -- & $+$ & -- & -- & -- & -- & -- \\
\bottomrule
\end{tabular}}
\caption{\textbf{Conceptual architecture and core-loss routing.}
Arrows show forward or auxiliary-supervision
paths, and the table marks parameter updates implied by the objectives
($+$: updated, --: not updated). The world model uses encoded replay
transitions and actions, not deployment-time lookahead. Alignment losses
detach encoder inputs and world-side targets, and actor losses do not update
critic parameters. Target-network copies are omitted for clarity, and all
parameters are frozen at deployment.}
\label{fig:architecture_core_losses}
\end{figure}
 
\begin{algorithm}[h]
\caption{Training Procedure (Core Objectives Shown)}
\label{alg:train_shielding}
\begin{algorithmic}[1]
\Require Batch of history-conditioned trajectories $\mathcal{B}$
\State Initialize shared encoder $E_\phi$, policy head $\pi_\theta$,
reward critic ensemble $\{\hat Q_{R,i}\}_{i=1}^M$, cost critic ensemble
$\{\hat Q_{C,i}\}_{i=1}^M$, and target copies
$\bar\pi_{\bar\theta}, \{\bar Q_{R,i}\}_{i=1}^M, \{\bar Q_{C,i}\}_{i=1}^M$
\State Initialize projection heads $g_{\omega}^{\mathrm{world}}$,
$g_{\psi}^{\mathrm{policy}}$
\State Initialize latent dynamics model $p_z$, reward head $\hat R$, and
cost heads $\hat C_1,\hat C_2$

\For{each training update}
    \State Sample batch $\mathcal{B}$ of padded trajectories indexed by $(k,t)$
    \State Encode transitions: $Z_t^k \gets E_\phi(H_t^k,S_t^k)$ and
    $Z_{t+1}^k \gets E_\phi(H_{t+1}^k,S_{t+1}^k)$
    \State Project to world and policy latents:
    \Statex \hspace{2em}
    $Z_t^{w,k} \gets g_{\omega}^{\mathrm{world}}(Z_t^k)$,\quad
    $Z_{t+1}^{w,k} \gets g_{\omega}^{\mathrm{world}}(Z_{t+1}^k)$
    \Statex \hspace{2em}
    $Z_t^{p,k} \gets g_{\psi}^{\mathrm{policy}}(Z_t^k)$,\quad
    $Z_{t+1}^{p,k} \gets g_{\psi}^{\mathrm{policy}}(Z_{t+1}^k)$

    \State Compute policy outputs $\pi_\theta(\cdot\mid Z_t^{p,k})$
    \State Compute world-model outputs:
    \Statex \hspace{2em}
    $f_z(Z_t^{w,k},A_t^k)$,\quad predicted reward $\hat R_{t+1}^k$,\quad
    predicted costs $\hat C_{1,t+1}^k,\hat C_{2,t+1}^k$

    \State Evaluate target-policy actions: enumerate discrete actions and
    sample once for continuous actions
    \State Compute detached reward/cost targets and normalized critic loss
    using Equation~\eqref{eq:training_critic_losses}

    \State Compute actor loss:
    \Statex \hspace{2em}
    $\mathcal{L}_{\mathrm{actor}}
    \gets
    \mathcal{L}_{\mathrm{DPG}}
    +
    \alpha_{\mathrm{BC}}\mathcal{L}_{\mathrm{AWBC}}
    +
    \lambda_C \mathcal{L}_{C}^{\pi}$
    \Comment{terms and masks: Equation~\eqref{eq:training_actor_losses}}

    \State Compute world-model loss $\mathcal{L}_{\mathrm{wm}}$
    \State Compute distillation loss $\mathcal{L}_{\mathrm{distill}}^{\mathrm{sg}}$
    \State Compute conjugacy loss $\mathcal{L}_{\mathrm{conj}}^{\mathrm{sg}}$

    \State $\mathcal{L}_{\mathrm{core}} \gets
    \mathcal{L}_{\mathrm{actor}}
    +
    10.0\,\mathcal{L}_{\mathrm{critic}}
    +
    1.0\,\mathcal{L}_{\mathrm{wm}}
    +
    0.1\,\mathcal{L}_{\mathrm{distill}}^{\mathrm{sg}}
    +
    0.1\,\mathcal{L}_{\mathrm{conj}}^{\mathrm{sg}}$

    \State Update trainable parameters (core objectives shown)
    \State Update target networks by Polyak averaging
\EndFor
\end{algorithmic}
\end{algorithm}
Here $\mathcal{L}_{\mathrm{DPG}}$ is the online policy-improvement loss,
$\mathcal{L}_{\mathrm{AWBC}}$ is the advantage-filtered behavioral cloning loss,
and $\mathcal{L}_{C}^{\pi}$ is the Lagrangian cost penalty computed from the
pessimistic cost critic. We use $M=4$ critics, one target
action sample, and $\alpha_{\mathrm{BC}}=0.1$. Bellman targets are computed with
Polyak-averaged target networks and detached from the critic update.
Discrete-action expectations are enumerated rather than sampled.

\textbf{Remark. }
SCARED uses the same base policy training pipeline as Q-Barrier, including the
actor/critic updates, target networks, target-conditioning protocol, and training
schedule. The difference is that Q-Barrier augments this backbone with
shield-specific modules and losses: the world/policy projections, latent
dynamics model, reward/cost prediction heads, structural
alignment losses, and the runtime Q-barrier action-selection rule.

\begin{algorithm}[h]
\caption{Q-Barrier Shielding for Finite Candidate Action Selection}
\label{alg:runtime_shielding}
\begin{algorithmic}[1]
\Require State $S_t^k$, context $H_t^k$, budget $\delta$, frozen policy $\pi_{\theta_*}$, scaled ensemble critics $\{\hat Q_{C,i}^{\mathrm{sc}}\}_{i=1}^M$, shield type $q\in\{\mathrm{soft},\mathrm{hard}\}$
\State Encode shared latent: $Z_t^k \gets E_\phi(H_t^k,S_t^k)$
\State Project latents: $Z_t^{w,k} \gets g_\omega^{\mathrm{world}}(Z_t^k)$,\;
$Z_t^{p,k} \gets g_\psi^{\mathrm{policy}}(Z_t^k)$
\State Compute remaining budget:
$B_t^k \gets \delta-\sum_{\tau=1}^{t} C_\tau^k$
\State Match critic units and adjust budget:
$T_t^k\gets[\alpha[s_C B_t^k]_+-m_{\mathrm{sh}}]_+$

\If{action space is discrete}
    \State $\mathcal A_{t,k}^{\mathrm{cand}} \gets \mathcal A$
    \State $\rho_t(A) \gets \pi_{\theta_*}(A\mid Z_t^{p,k})$ for each $A\in\mathcal A_{t,k}^{\mathrm{cand}}$
\Else
    \State Sample $\mathcal A_{t,k}^{\mathrm{cand}}=\{A^{(j)}\}_{j=1}^{N_s}$ from $\pi_{\theta_*}(\cdot\mid Z_t^{p,k})$
    \State $\rho_t(A) \gets 1$ for each $A\in\mathcal A_{t,k}^{\mathrm{cand}}$
\EndIf

\For{each $A\in\mathcal A_{t,k}^{\mathrm{cand}}$}
    \State $\hat Q_C^{\mathrm{sc},+}(Z_t^{w,k},A) \gets \max_i \hat Q_{C,i}^{\mathrm{sc}}(Z_t^{w,k},A)$
    \State $\tilde b^k_{Q,t}(A)
    \gets T_t^k-\hat Q_C^{\mathrm{sc},+}(Z_t^{w,k},A)$
\EndFor

\If{$q=\mathrm{soft}$}
    \For{each $A\in\mathcal A_{t,k}^{\mathrm{cand}}$}
        \State $w_t(A)\gets
        \rho_t(A)\exp\!\left(-[-\tilde b^k_{Q,t}(A)]_+\right)$
    \EndFor
    \State Normalize:
    $\pi_{\mathrm{soft}}(A)\gets
    w_t(A)\big/\sum_{A'\in\mathcal A_{t,k}^{\mathrm{cand}}}w_t(A')$
    \State \Return $A_t^k\sim \pi_{\mathrm{soft}}$
\Else
    \State $\mathcal A_{\mathrm{safe},t}^k
    \gets
    \{A\in\mathcal A_{t,k}^{\mathrm{cand}}:
    \tilde b^k_{Q,t}(A)\ge 0\}$

    \If{$\mathcal A_{\mathrm{safe},t}^k\neq\emptyset$}
        \For{each $A\in\mathcal A_{t,k}^{\mathrm{cand}}$}
            \State $w_t(A)\gets
            \rho_t(A)\mathbb{I}[A\in\mathcal A_{\mathrm{safe},t}^k]$
        \EndFor
        \State Normalize:
        $\pi_{\mathrm{hard}}(A)\gets
        w_t(A)\big/\sum_{A'\in\mathcal A_{t,k}^{\mathrm{cand}}}w_t(A')$
        \State \Return $A_t^k\sim \pi_{\mathrm{hard}}$
    \Else
        \State $\mathcal A_{\min,t}^k
        \gets
        \arg\min_{A\in\mathcal A_{t,k}^{\mathrm{cand}}}
        \hat Q_C^{\mathrm{sc},+}(Z_t^{w,k},A)$
        \State \Return $A_t^k\sim \mathrm{Uniform}(\mathcal A_{\min,t}^k)$
        \Comment{lowest predicted-cost fallback}
    \EndIf
\EndIf
\end{algorithmic}
\end{algorithm}
\paragraph{Notes on presentation.}
Here $\hat Q_{C,i}^{\mathrm{sc}}=s_C\hat Q_{C,i}$ is the PopArt-denormalized
prediction in critic-scaled cost units. The fixed budget factor is
$\alpha=1.25$ for SafeDarkRoom/SafeVelocity and $1.5$ for SafeDarkMujoco,
with $m_{\mathrm{sh}}=0$. The deployed margin $\tilde b_Q$ is distinct from
the signed margin $b_Q$ analyzed in Theorem~\ref{thm:barrier_margin_propagation}.
Algorithm~\ref{alg:train_shielding} expands the base history-conditioned actor-critic forward path, making the interaction between policy optimization and auxiliary shielding objectives explicit.
Algorithm~\ref{alg:runtime_shielding} is the deployment-time rule used by the proposed method: the shield encodes the current state and context into shared and projected latents, computes the remaining budget, evaluates candidate actions with a pessimistic ensemble cost critic, and biases action selection toward actions with nonnegative barrier margin.

\section{Environment Details}
\label{app:env_details}
\begin{figure*}[h]
\centering
\begin{tikzpicture}
  \matrix (m) [matrix of nodes,
             nodes={inner xsep=5pt, inner ysep=4pt, outer sep=0, anchor=center},
             row sep=8mm] {
    \includegraphics[width=.25\linewidth]{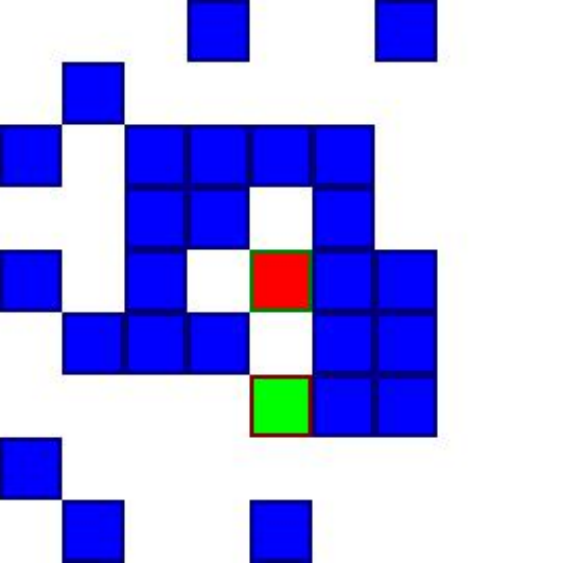} &
    \includegraphics[width=.25\linewidth]{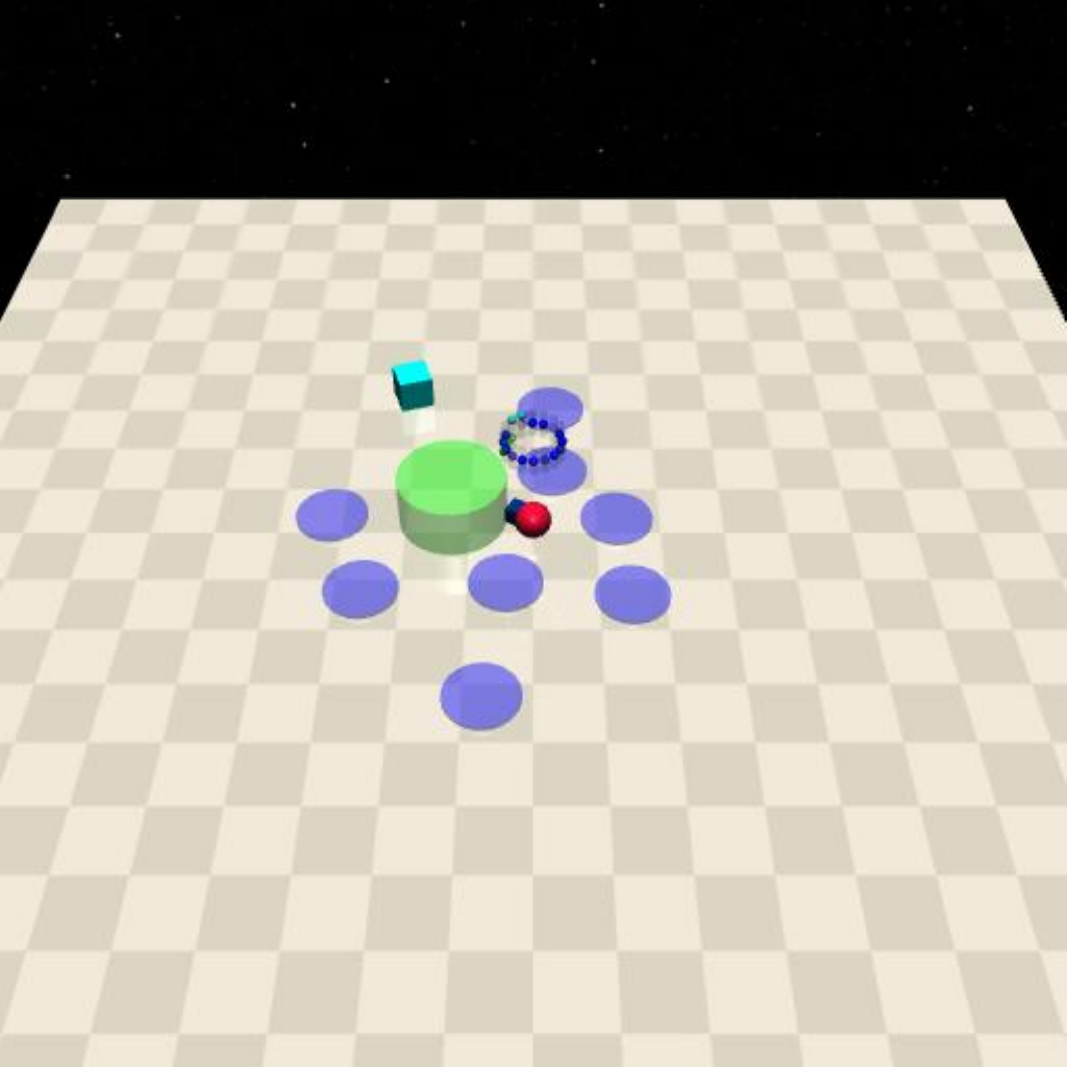} &
    \includegraphics[width=.25\linewidth]{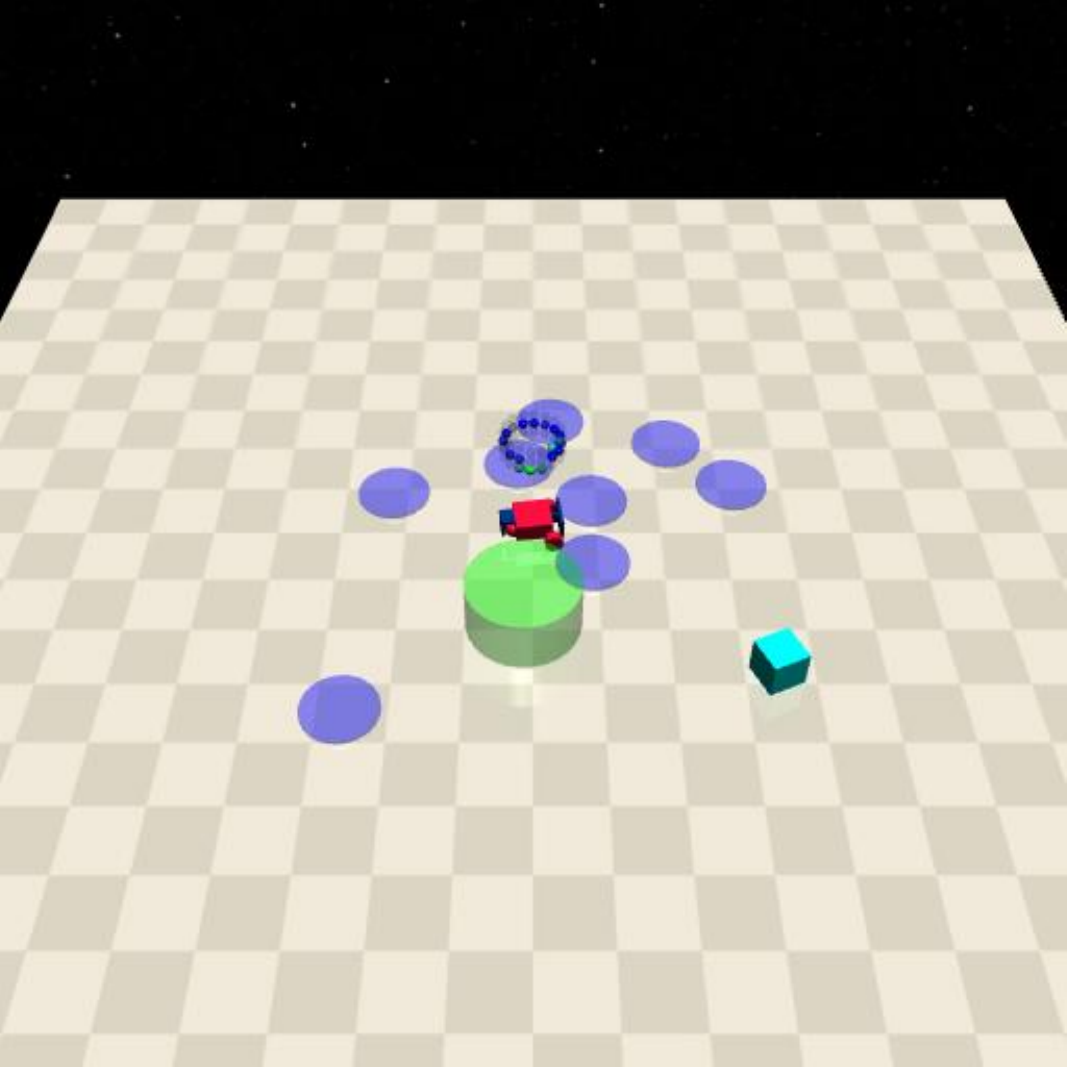} \\
    \includegraphics[width=.25\linewidth]{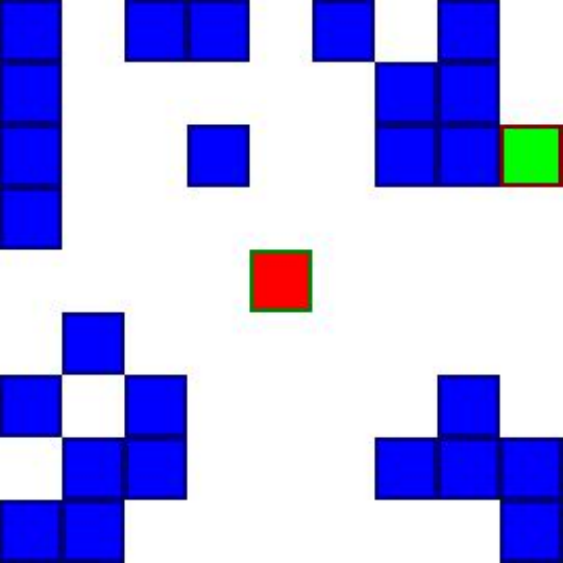} &
    \includegraphics[width=.25\linewidth]{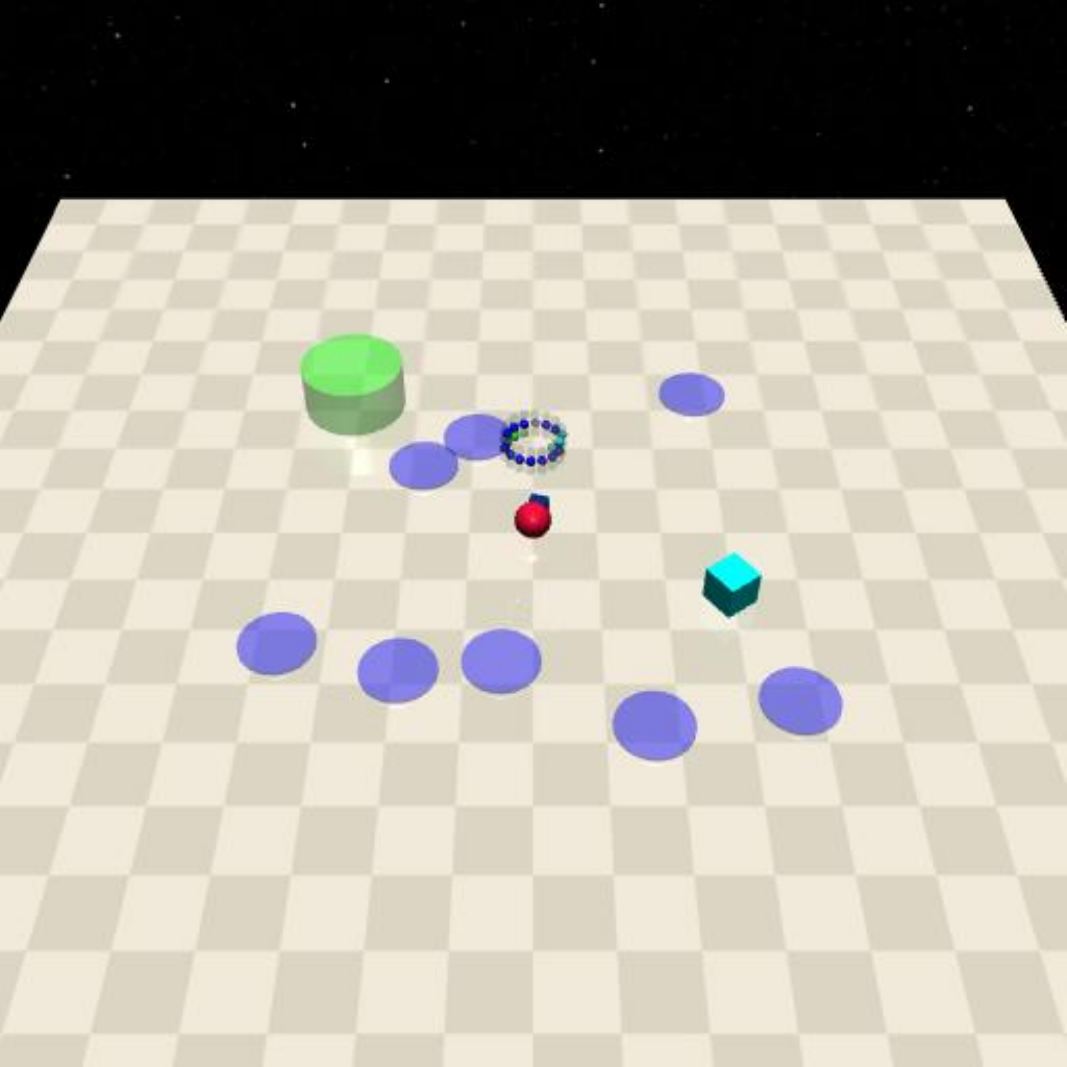} &
    \includegraphics[width=.25\linewidth]{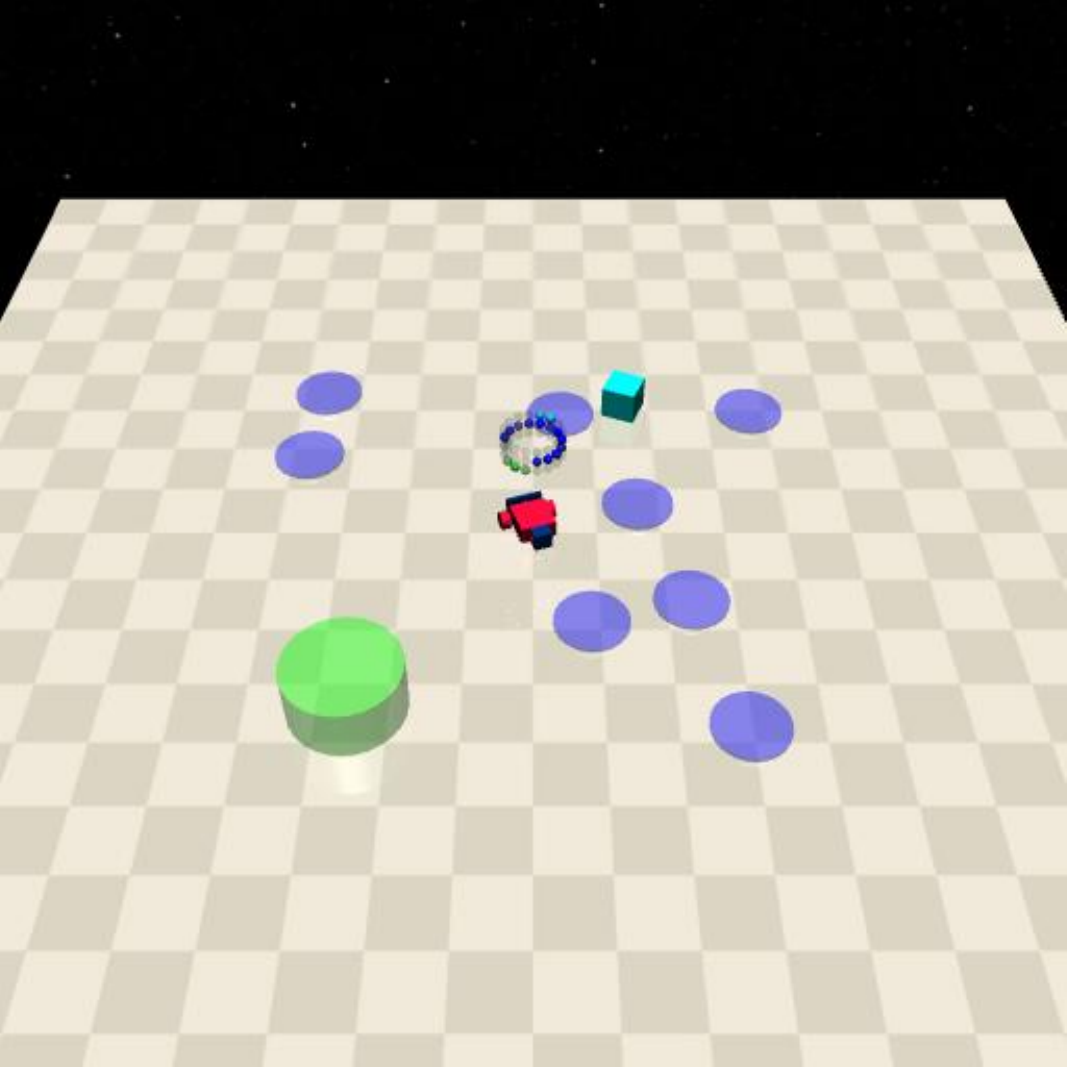} \\
  };

\path
  let \p2 = (m-1-3.south east) in coordinate (topRight) at (\x2,\y2);

\path
  let \p1 = (m-1-1.north west) in coordinate (topLeft) at (\x1,\y1);
\path
  let \p3 = (m-2-1.north west) in coordinate (botLeft) at (\x3,\y3);

\path
  let \p4 = (m-2-3.south east) in coordinate (botRight) at (\x4,\y4);

\begin{pgfonlayer}{background}
  \node (topbg) [fit=(topLeft) (topRight),
                 fill=ao!12, inner sep=0, rounded corners=2pt] {};
  \node (botbg) [fit=(botLeft) (botRight),
                 fill=bittersweet!12,  inner sep=0, rounded corners=2pt] {};
\end{pgfonlayer}

  \coordinate (YL) at ($(topbg.south)!0.5!(botbg.north)$);

  \coordinate (TopMidY) at ($(topbg.north)!0.5!(topbg.south)$);
  \coordinate (BotMidY) at ($(botbg.north)!0.5!(botbg.south)$);

\node[anchor=north, inner sep=0pt,
      font=\bfseries\fontsize{9}{9.6}\selectfont, text=ao!60!black]
     at ([yshift=-2mm]topbg.south)
     {(a) Training Environments};
\node[anchor=north, inner sep=0pt,
      font=\bfseries\fontsize{9}{9.6}\selectfont, text=bittersweet!60!black]
     at ([yshift=-2mm]botbg.south)
     {(b) Test Environments};
\tikzset{colname/.style={font=\bfseries\small, inner sep=0pt}}

\coordinate (C1) at (m-1-1.center);
\coordinate (C2) at (m-1-2.center);
\coordinate (C3) at (m-1-3.center);

\node[colname, anchor=south] at ([yshift=1mm] topbg.north -| C1) {SafeDarkRoom};
\node[colname, anchor=south] at ([yshift=1mm] topbg.north -| C2) {SafeDarkMujoco-Point};
\node[colname, anchor=south] at ([yshift=1mm] topbg.north -| C3) {SafeDarkMujoco-Car};

\end{tikzpicture}
\caption{\label{fig: env} Visualization of the environment layouts, directly adapted from \citet{moeini2026scared}. The agent (red) must navigate to the target location (green) while avoiding obstacles (blue shades). To evaluate out-of-distribution generalization, training configurations sample goals and obstacles using a center-oriented distribution ($\alpha$=0.5), whereas evaluation configurations strictly employ edge-oriented placements.}
\end{figure*}
\paragraph{Out-of-Distribution Environment Generation.}
For SafeDarkRoom and SafeDarkMujoco, we adopt the distance-based spawning protocol used in the safe ICRL benchmark of \citet{moeini2026scared}.
This benchmark creates a controlled distributional shift by sampling training and evaluation environments from spatial distributions with opposite concentration patterns.

During training, obstacle and goal locations are sampled from a center-concentrated distribution,
\[
p_{\text{train}}(x) \propto \exp(-\alpha d(x,c)),
\]
where $d(x,c)$ denotes Euclidean distance from the map center $c$.
Evaluation environments invert this distribution,
\[
p_{\text{test}}(x) \propto \exp(\alpha d(x,c)),
\]
forcing the agent to extrapolate into regions that are rarely observed during training.
As discussed by \citet{moeini2026scared}, increasing $\alpha$ sharpens the separation between training and test distributions and makes the resulting shift increasingly severe.
We refer readers to that work for the full benchmark construction and derivations.
For convenience, Figure~\ref{fig: env} reproduces representative training and evaluation layouts.

\paragraph{Measuring OOD Generalization.}
Goal-discovery environments such as DarkRoom are widely used in prior ICRL work \citep{laskin2023incontext, zisman2023emergence, son2025distilling}.
In those settings, test-time generalization is typically evaluated by holding out goal locations during training.
However, such held-out tasks may still be solvable by interpolation across nearby training tasks \citep{kirkSurveyZeroshotGeneralisation2023}.

The structural OOD setup used here is stricter.
Rather than withholding individual goals, we change the spatial distribution from which both goals and obstacles are sampled.
During training, the agent starts near the map center and encounters center-oriented tasks more often.
During evaluation, both goals and obstacles are drawn from edge-oriented distributions.
This shift applies to both SafeDarkRoom and SafeDarkMujoco and is designed to test extrapolative generalization to regions that are systematically underrepresented during training.

\paragraph{Safe ICRL Benchmarks.}
SafeDarkRoom is a partially observed grid-world adaptation of DarkRoom \citep{laskin2023incontext}.
The agent observes only its own position and cannot directly observe either the goal or the obstacles.
Rewards are sparse, and costs arise from stepping on hazards distributed throughout the map.
The agent must therefore infer both where to go and where not to go from the reward and cost signals accumulated in context.

SafeDarkMujoco is the continuous-control counterpart built on MuJoCo \citep{todorov2012mujoco}.
The robot observes internal physical states such as position, velocity, acceleration, and rotation angle, but does not receive direct sensing of the goal or obstacles.
As in SafeDarkRoom, successful adaptation requires using sparse reward and cost observations from previous interaction to infer the hidden task structure.

SafeVelocity follows the Safety Gym style of safe locomotion tasks \citep{ji2023safety}, but evaluation is performed on held-out target velocities from the same task family.
We treat this benchmark as unseen ID generalization rather than structural OOD generalization because the deployment tasks are new parameter settings within the same family, not tasks sampled from a qualitatively shifted spatial distribution.

\section{Training Details}
\label{app:training_details}
For the main Q-Barrier/SCARED comparisons, $K=50$ in SafeDarkRoom, $15$ in
Point/Car, and $10$ in SafeVelocity.

\paragraph{SCARED.} We use the original implementation of~\citet{moeini2026scared}. During SCARED pretraining, we resample a new environment from the training distribution every $K$ episodes.
For each sampled environment, we also sample a CTG target.
The CTG range is $[1,15]$ for SafeDarkRoom and $[10,50]$ for SafeDarkMujoco.

SCARED follows the AMAGO-style history-conditioned architecture of \citet{grigsby2023amago}.
An MLP time-step encoder maps each tuple $(S_t, A_t, R_t, C_t)$ to an embedding, and a transformer-based trajectory encoder processes the resulting sequence. The prediction head outputs either an action distribution in discrete-action environments or a value prediction in continuous-action environments.
Table~\ref{tab:scared_params} lists the remaining hyperparameters.

\paragraph{Q-Barrier.}
Shield-supporting modules, including the encoder $E_\phi$, projection heads $g_\omega^{\mathrm{world}}$ and $g_\psi^{\mathrm{policy}}$, latent dynamics model $p_z$, and ensemble cost critic $\{\hat Q_{C,i}\}_{i=1}^M$, train jointly with SCARED during pretraining, without additional epochs.
Discrete-action training collection uses the soft shield over the full action space.
Continuous-action collection uses $N_s=1$, so reweighting cannot change the
base-policy action, but representation and critic learning remain active.
At test time, the frozen continuous-action shield uses $N_s=8$ by default and
applies Algorithm~\ref{alg:runtime_shielding}.
The core loss weights in Equation~\eqref{eqn:q-barrier loss} are fixed across all environments:
$\lambda_{\mathrm{critic}}=10.0$,
$\lambda_{\text{wm}}=1.0$,
$\lambda_{\text{dist}}=0.1$, and
$\lambda_{\text{conj}}=0.1$.
The remaining backbone hyperparameters follow Table~\ref{tab:scared_params}.

\begin{table*}[t]
\centering
\begin{tabular}{|l|l|l|}
\hline
\textbf{Parameter} & \textbf{SafeDarkRoom} & \textbf{SafeDarkMujoco \& SafeVelocity} \\
\hline
Episode time limit $t_{\max}$ & 30 & 75 (navigation), 200 (velocity) \\
\hline
Replay buffer capacity & 100,000 & 100,000 \\
\hline
Embedding dim & 64 & 64 \\
\hline
Hidden dim & 64 & 64 \\
\hline
Num layers & 4 & 4 \\
\hline
Num heads & 8 & 8 \\
\hline
Seq len & 1500 & 1500 (navigation), 2000 (velocity) \\
\hline
Attention dropout & 0 & 0 \\
\hline
Residual dropout & 0 & 0 \\
\hline
Learning rate & $3\times 10^{-4}$ & $3\times 10^{-4}$ \\
\hline
Betas & $(0.9, 0.99)$ & $(0.9, 0.99)$ \\
\hline
Clip grad & 1.0 & 1.0 \\
\hline
Batch size & 32 & 32 \\
\hline
Optimizer & Adam & Adam \\
\hline
\end{tabular}
\caption{Hyperparameters for SCARED. Navigation uses 75 policy steps.
Velocity uses 200 environment steps.}
\label{tab:scared_params}
\end{table*}

\paragraph{SafeMeta and MAML with penalty.}
We use the original implementations of \citet{xu2025efficient} with three hidden layers of sizes $(64, 512, 64)$ for the policy, value, and cost networks, for a total of 205,510 parameters.
We train for 15k meta-training iterations, sampling 20 tasks from the training distribution at each iteration.
During meta-testing, both algorithms perform $K$ gradient updates to adapt to new tasks with unseen cost-to-go values.

\begin{table*}[t]
\centering
\begin{tabular}{|l|l|l|}
\hline
\textbf{Parameter} & \textbf{SafeDarkRoom} & \textbf{SafeDarkMujoco \& SafeVelocity} \\
\hline
Episode time limit $t_{\max}$ & 30 & 75 (navigation), 200 (velocity) \\
\hline
Hidden layers & $(64, 512, 64)$ & $(64, 512, 64)$ \\
\hline
Min/max batch size & 500/1500 & 1500/1500 \\
\hline
Policy learning rate & $10^{-3}$ & $10^{-3}$ \\
\hline
Value/cost learning rate & $3\times 10^{-2} / 10^{-1}$ & $3\times 10^{-2} / 10^{-1}$ \\
\hline
Discount factor $\gamma$ & 0.99 & 0.99 \\
\hline
GAE parameter $\tau$ & 0.95 & 0.95 \\
\hline
Max KL divergence & $10^{-3}$ & $10^{-3}$ \\
\hline
Lagrangian weight $\lambda$ & 1.0 & 1.0 \\
\hline
Num meta-training iterations & 15k & 15k \\
\hline
Policy optimizer & Adam & Adam \\
\hline
Value/cost optimizer & L-BFGS & L-BFGS \\
\hline
\end{tabular}
\caption{Hyperparameters for SafeMeta and MAML with penalty. Episode-step
conventions match Table~\ref{tab:scared_params}.}
\label{tab:safemeta_maml_parameters}
\end{table*}

\paragraph{Safe Algorithm Distillation.}
We use the original implementation of~\citet{moeini2026scared}. For Safe AD, we collect a dataset $\mathcal{D}=\{\Xi_i\}$ of learning histories, where each trajectory $\Xi_i \doteq (\tau_1,\tau_2,\dots,\tau_K)$ is a sequence of episodes generated by running safe RL algorithms on CMDPs.
Each episode $\tau_k$ contains states, actions, rewards, and costs.

We train the Safe AD policy autoregressively to distill the behavior present in these logged learning histories, following \citet{laskin2023incontext}.
Each training example contains a trajectory segment together with RTG and CTG conditioning.
The policy takes as input $(S_t^k, H_t^k, G_t(\tau_k), G_{c,t}(\tau_k))$ and predicts the action at time $t$ in episode $k$.

\textbf{Discrete action spaces.}
For environments with discrete actions, the model outputs categorical logits and is trained with cross-entropy:
\begin{align}
\mathcal{L}_{\mathrm{disc}}(\theta)
&= \mathbb{E}_{\Xi_i \sim \mathcal{D}}
\left[
-\log \pi_\theta\!\left(A_t^k \mid S_t^k, H_t^k, G_t(\tau_k), G_{c,t}(\tau_k)\right)
\right].
\label{eq:cad_ce}
\end{align}

\textbf{Continuous action spaces.}
For environments with continuous actions, the model predicts an action mean and is trained with an $\ell_2$ regression loss:
\begin{align}
\mathcal{L}_{\mathrm{cont}}(\theta)
&= \mathbb{E}_{\Xi_i \sim \mathcal{D}}
\left[
\left\| A_t^k - \mu_\theta\!\left(S_t^k, H_t^k, G_t(\tau_k), G_{c,t}(\tau_k)\right) \right\|_2^2
\right].
\label{eq:cad_l2}
\end{align}

\textbf{Dataset collection.}
We follow the same training protocol from~\citep{moeini2026scared}. We collect Safe AD data by running PPO-Lagrangian \citep{schulman2017proximal, ray2019benchmarking} on the training environments.
For SafeDarkRoom, we vary the cost limit across $\{0, 2.5, 5.0\}$ and collect 50,000 steps of learning history for each setting.
For SafeDarkMujoco and SafeVelocity, we use larger datasets of 2--5 million environment steps to capture a broader range of safety-relevant behavior in continuous control.
In these continuous-control environments, we keep the cost limit fixed at $0$ during data collection so that the logged histories cover the full trajectory from high-cost exploration to near-feasible behavior later in training.

\begin{table*}[t]
\centering
\begin{tabular}{|l|l|l|}
\hline
\textbf{Parameter} & \textbf{SafeDarkRoom} & \textbf{SafeDarkMujoco \& SafeVelocity} \\
\hline
Embedding dim & 64 & 512 \\
\hline
Hidden dim & 512 & 256 \\
\hline
Num layers & 8 & 8 \\
\hline
Num heads & 8 & 8 \\
\hline
Seq len & 100 & 200 \\
\hline
Attention dropout & 0.5 & 0.5 \\
\hline
Residual dropout & 0.1 & 0.1 \\
\hline
Embedding dropout & 0.3 & 0.3 \\
\hline
Learning rate & $3\times 10^{-4}$ & $3\times 10^{-4}$ \\
\hline
Betas & $(0.9, 0.99)$ & $(0.9, 0.99)$ \\
\hline
Clip grad & 1.0 & 1.0 \\
\hline
Batch size & 512 & 128 \\
\hline
Optimizer & Adam & Adam \\
\hline
\end{tabular}
\caption{Hyperparameters for Safe Algorithm Distillation.}
\label{tab:safe_ad_params}
\end{table*}

\paragraph{Compute resources.}
All experiments were run on NVIDIA A100 (80GB) or H100 (80GB) GPUs. SCARED pretraining uses 3K epochs for SafeDarkRoom and SafeDarkMujoco and 2K for SafeVelocity, with the shield-supporting modules (encoder, projection heads, latent dynamics, ensemble cost critic) trained jointly during the same pretraining phase. Safe AD data collection via PPO-Lagrangian requires 2--5 million environment steps per environment. Safe meta-RL baselines (SafeMeta, MAML with penalty) train for up to 15K meta-training iterations. 

\subsection{Core Loss Definitions}
\label{app:loss_details}

We detail the terms in Equation~\eqref{eqn:q-barrier loss}. The reward
actor--critic update, filtered behavioral cloning, PopArt normalization, and
multi-discount training follow AMAGO~\citep{grigsby2023amago}, and the cost objectives
extend this backbone following SCARED~\citep{moeini2026scared}.
We suppress the episode index and write $z_t$ for a head's latent input.
Expectations average over nonpadding replay transitions and, for actor--critic
terms, the trained discount heads. We write one discount $\gamma$ explicitly, while
deployment uses $\gamma_*=0.999$.

\paragraph{Critic losses.}
Let $\hat Q_{h,i}$ be critic $i$ for $h\in\{R,C\}$ in scaled reward/cost units,
and let $\bar Q_{h,i}$ and $\bar\pi$ denote Polyak-averaged target heads.
Here $R,C,\delta$ are before critic scaling, and the locally denoted scaled cost head
equals $s_C$ times the critic used in the barrier analysis.
Write $\mathcal N_h(x)=(x-\mu_h)/\sigma_h$ for PopArt normalization and
$q_{h,i}=\mathcal N_h(\hat Q_{h,i})$ for the normalized prediction used for
optimization. With immediate reward/cost scales $s_R,s_C$, the targets are
\begin{equation}
\begin{aligned}
Y_t^R &= s_R R_{t+1}
+\gamma(1-D_{t+1})\min_{j\in J_R}
\mathbb E_{a'\sim\bar\pi(\cdot\mid z_{t+1})}
 [\bar Q_{R,j}(z_{t+1},a')],\\
Y_t^C &= s_C C_{t+1}
+\gamma(1-D^{\mathrm{ep}}_{t+1})\max_{j\in J_C}
\mathbb E_{a'\sim\bar\pi(\cdot\mid z_{t+1})}
 [\bar Q_{C,j}(z_{t+1},a')],\\
\mathcal L_{\mathrm{critic}} &=
\mathbb E_t\!\left[\frac{1}{M}\sum_{i=1}^{M}
\sum_{h\in\{R,C\}}
 \bigl(q_{h,i}(z_t,a_t)-\operatorname{sg}(\mathcal N_h(Y_t^h))\bigr)^2\right].
\end{aligned}
\label{eq:training_critic_losses}
\end{equation}
Here $a_t$ is the replay action, $D$ is the replay trajectory's termination
flag, and $D^{\mathrm{ep}}$ is its inner-episode reset flag. Cost bootstrapping
therefore stops at each episode boundary. Each $J_h$ contains two critic
indices sampled with replacement. Reward minima limit overestimation, and cost
maxima discourage underestimation. These targets use the \emph{observed} next
encoded state, not a latent-model rollout. Target evaluation and PopArt
statistic updates are detached. The TD losses update the online critics and
encoder. Target-head Polyak averaging uses $\tau=0.003$.

\paragraph{Actor losses.}
Define each normalized critic's policy value and the replay-action advantage as
\begin{equation}
v_{h,i}(z_t)=\mathbb E_{a\sim\pi_\theta(\cdot\mid z_t)}[q_{h,i}(z_t,a)],
\qquad
A_t=\frac{1}{M}\sum_i\bigl(q_{R,i}(z_t,a_t)-v_{R,i}(z_t)\bigr).
\label{eq:training_actor_values}
\end{equation}
The actor terms are
\begin{equation}
\begin{aligned}
\mathcal L_{\mathrm{DPG}} &=-\mathbb E_t[u_t\min_i v_{R,i}(z_t)],\\
\mathcal L_{\mathrm{AWBC}} &=-\mathbb E_t[
 \operatorname{sg}(\mathbf 1\{A_t>0\})\log\pi_\theta(a_t\mid z_t)],\\
\mathcal L_C^\pi &=\mathbb E_t[u_t m_t\max_i v_{C,i}(z_t)],\qquad
\mathcal L_{\mathrm{actor}}=
\mathcal L_{\mathrm{DPG}}+0.1\mathcal L_{\mathrm{AWBC}}
+\lambda_C\mathcal L_C^\pi .
\end{aligned}
\label{eq:training_actor_losses}
\end{equation}
DPG improves reward using a conservative critic value. AWBC stabilizes
off-policy learning by cloning only replay actions with positive estimated
advantage. Its binary filter is detached. The cost term discourages high
predicted cost on the transitions selected by the training masks.
The mask $m_t$ maps episode-cost labels to replay transitions: a label is $+1$
for total cost above $\delta$, $-1$ for total cost below $(1-\eta)\delta$, and
zero otherwise, where $\eta$ is the maximum unused budget fraction. The optional detached mask $u_t\in\{0,1\}$ excludes online
reward and cost terms when the minimum cost-critic value at an episode's start,
using $\gamma_*$ and scaled cost units, exceeds $s_R\delta$, otherwise $u_t=1$.
AWBC uses neither mask.

For discrete actions, expectations are exact sums over policy probabilities,
and for continuous actions, they use one reparameterized sample. The critic min/max
is taken \emph{after} this expectation. Actor gradients pass through the
proposed action or probabilities. Actor optimization does not update critic
parameters, and the critic's latent input is detached. The policy-input path
still updates the encoder.
Continuous-action BC log probabilities are summed over action dimensions.

\paragraph{World-model loss.}
Let $w_t=g_\omega^{\mathrm{world}}(Z_t)$. The latent model predicts a mean
$\mu_t=f_z(w_t,a_t)$ and coordinate log-variances $\ell_{tj}$, clipped to
$[-10,5]$. With $d_m$ coordinates and two immediate-cost heads,
\begin{equation}
\mathcal L_{\mathrm{wm}}=\mathbb E_t\!\left[
\frac{1}{2d_m}\sum_{j=1}^{d_m}
 \left(e^{-\ell_{tj}}(\mu_{tj}-w_{t+1,j})^2+\ell_{tj}\right)
+(\hat R_{t+1}-R_{t+1})^2
+\sum_{i=1}^{2}(\hat C_{i,t+1}-C_{t+1})^2\right].
\label{eq:training_world_loss}
\end{equation}
The first term is the coordinate-averaged Gaussian negative log likelihood,
omitting its constant. The other terms predict observed reward and cost without
the critic's scaling factors, and each cost head has its own regression loss.
Both current and next world latents remain differentiable. These objectives
train the encoder, world projection, and prediction heads to retain
transition-relevant information, without model-based planning at deployment.

\paragraph{Structural alignment losses.}
Distillation matches projected states. Conjugacy matches their discounted
temporal changes. Let $p_t=g_\psi^{\mathrm{policy}}(\operatorname{sg}(Z_t))$.
The reduced losses are
\begin{equation}
\begin{aligned}
\mathcal L_{\mathrm{distill}}^{\mathrm{sg}}
 &=\mathbb E_t\!\left[d_m^{-1}\|p_t-\operatorname{sg}(w_t)\|_2^2\right],\\
\mathcal L_{\mathrm{conj}}^{\mathrm{sg}}
 &=\mathbb E_t\!\left[d_m^{-1}
 \|p_{t+1}-\gamma_*p_t
  -\operatorname{sg}(w_{t+1}-\gamma_*w_t)\|_2^2\right].
\end{aligned}
\label{eq:training_alignment_losses}
\end{equation}
Conjugacy encourages consistent descriptions of an observed transition, but it
does not predict the next state. Detaching the encoder inputs and world-side
targets makes both losses update only the policy projection.

\paragraph{Loss weights and implementation scope.}
The core weights in Equation~\eqref{eqn:q-barrier loss} are
$(\lambda_{\mathrm{critic}},\lambda_{\mathrm{wm}},\lambda_{\mathrm{dist}},
\lambda_{\mathrm{conj}})=(10,1,0.1,0.1)$.

 \section{Additional Shielding Ablations}
\label{app:additional_ablations}

\subsection{Component Ablations}
\label{app:auxiliary_loss_checks}

Table~\ref{tab:auxiliary_loss_checks} summarizes the contributions of
dynamics-based representation learning, runtime shielding, and the cost-critic
ensemble on SafeDarkRoom. The latent dynamics model is used only during training
to shape the representation used by the cost critic, not for model rollouts at
deployment.

\begin{table}[htbp]
\centering
\begin{tabular}{l rr}
\toprule
Method & Average Return $\uparrow$ & Average Cost $\downarrow$ \\
\midrule
Full Q-Barrier ($M=4$) & $0.898\pm0.003$ & $0.920\pm0.034$ \\
Two critics ($M=2$) & $0.836\pm0.006$ & $1.360\pm0.045$ \\
Single critic ($M=1$) & $0.770\pm0.011$ & $1.240\pm0.034$ \\
Barrier without dynamics learning & $0.770\pm0.009$ & $1.747\pm0.041$ \\
Dynamics learning without barrier & $0.808\pm0.007$ & $2.002\pm0.066$ \\
SCARED & $0.625\pm0.012$ & $1.660\pm0.010$ \\
\bottomrule
\end{tabular}
\caption{Component ablations on SafeDarkRoom: mean $\pm$
standard error of the mean over 50 in-context episodes and 100 evaluation
rollouts. 
The no-dynamics variant removes the associated auxiliary
representation objectives.}
\label{tab:auxiliary_loss_checks}
\end{table}

Full Q-Barrier has the highest return and lowest cost. Even one critic
outperforms SCARED in both metrics, so ensembling alone does not explain the
improvement. Increasing $M$ improves return, although cost is not monotonic in
ensemble size. Removing dynamics-based representation learning or the barrier
worsens both metrics relative to the full model, supporting their complementary
roles.

\subsection{Comparison with Cost-Penalty Reweighting}
\label{app:cost_penalty_ablation}

We compare our barrier penalty with a simpler rule that penalizes predicted cost
no matter how much budget is left. Using the scaled critic and threshold from
Algorithm~\ref{alg:runtime_shielding}, the two SafeDarkRoom rules weight each action as
\[
\begin{aligned}
w_{\mathrm{barrier}}(a)
&=\pi_\theta(a\mid z_t)\exp\{-\beta[\hat Q_C^{\mathrm{sc},+}(z_t,a)-T_t]_+\},\\
w_{\mathrm{penalty}}(a)
&=\pi_\theta(a\mid z_t)\exp\{-\beta\hat Q_C^{\mathrm{sc},+}(z_t,a)\},
\end{aligned}
\]
and the weights are normalized over actions. Each model is trained with its own
rule at $\beta=1$, using only the base actor--critic objectives. We then vary
$\beta$ at evaluation time without retraining.

\begin{table}[htbp]
\centering
\begin{tabular}{r rr rr}
\toprule
& \multicolumn{2}{c}{Cost penalty} & \multicolumn{2}{c}{Barrier} \\
\cmidrule(lr){2-3}\cmidrule(lr){4-5}
$\beta$ & Return $\uparrow$ & Cost $\downarrow$
& Return $\uparrow$ & Cost $\downarrow$ \\
\midrule
0.1  & 0.647 & 2.043 & 0.780 & 1.816 \\
0.5  & 0.644 & 1.972 & 0.765 & 1.800 \\
1.0  & 0.633 & 1.833 & 0.769 & 1.750 \\
2.0  & 0.590 & 1.633 & 0.593 & 1.489 \\
5.0  & 0.434 & 1.379 & 0.382 & 1.057 \\
10.0 & 0.128 & 0.859 & 0.218 & 0.948 \\
\bottomrule
\end{tabular}
\caption{Cost-penalty versus barrier reweighting on SafeDarkRoom, in the same
return and cost units as the main results.}
\label{tab:cost_penalty_ablation}
\end{table}

At $\beta=1$, the barrier gets higher return (0.769 versus 0.633) and slightly
lower cost (1.750 versus 1.833). A larger $\beta$ lowers both return and cost
for either rule, so the two metrics have to be read together, and neither rule
is best at every $\beta$. The barrier does not penalize actions whose predicted
cost fits in the remaining budget, while the cost penalty keeps favoring low
cost even when the budget is large.

\subsection{Executed-Action Barrier Margins}
\label{app:negative_margin_selections}

The soft shield can execute an action with negative deployment margin
$\tilde b_Q$ (Algorithm~\ref{alg:runtime_shielding}).
Table~\ref{tab:negative_margin_selections} reports how often this occurs,
separately over all decisions and those with nonnegative remaining budget.

\begin{table}[htbp]
\centering
\begin{tabular}{l rr}
\toprule
Benchmark & \shortstack{Negative executed margin\\all decisions}
& \shortstack{Negative executed margin\\when $B_t\ge0$} \\
\midrule
SafeDarkRoom      & $1.6\%$  & $0.9\%$ \\
SafetyPoint       & $1.3\%$  & $0.2\%$ \\
SafetyCar         & $2.0\%$  & $0.4\%$ \\
SafetyAnt         & $4.5\%$  & $2.7\%$ \\
SafetyHalfCheetah & $25.8\%$ & $22.5\%$ \\
\bottomrule
\end{tabular}
\caption{Executed-action negative-margin frequencies, using the implemented margin $\tilde b_Q$. DarkRoom, Point, and Car
each use 15,000 logged decisions. Ant and HalfCheetah each use 2,000,000.}
\label{tab:negative_margin_selections}
\end{table}

HalfCheetah remains the outlier even with nonnegative remaining budget.

\subsection{Candidate-Set Feasibility}
\label{app:candidate_feasibility_ablation}

We also check how often at least one sampled candidate passes the barrier
($\tilde b_Q\ge0$ in Algorithm~\ref{alg:runtime_shielding}), among decisions
with nonnegative remaining budget ($B_t\ge0$).
Table~\ref{tab:candidate_feasibility_ablation} varies the number of proposals
$N_s$ from 1 to 16. With $N_s=1$, reweighting cannot change the action.

\begin{table}[htbp]
\centering
\begin{tabular}{l rrrr}
\toprule
Benchmark & $N_s=1$ & $N_s=4$ & $N_s=8$ & $N_s=16$ \\
\midrule
SafetyPoint       & 99.90 & 99.97 & 99.91 & 99.89 \\
SafetyCar         & 99.51 & 99.05 & 99.65 & 99.98 \\
SafetyAnt         & 97.12 & 97.28 & 97.38 & 97.47 \\
SafetyHalfCheetah & 78.52 & 79.69 & 79.29 & 79.49 \\
\bottomrule
\end{tabular}
\caption{Percentage of decisions with at least one barrier-passing candidate,
among decisions with $B_t\ge0$.}
\label{tab:candidate_feasibility_ablation}
\end{table}

The percentage barely changes with $N_s$. SafetyHalfCheetah is the only exception showing
that about 20 percent of decisions have no barrier-passing candidate, and more candidates do
not fix this even though the number slightly increases. This is consistent with the conservative behavior of Q-Barrier in SafetyHalfCheetah.

\subsection{Training Context Length}
\label{app:training_context_ablation}

We compare Q-Barrier trained with 300- and 1,500-step context sequences
(approximately 10 and 50 episodes). Both are evaluated on SafeDarkRoom for
50 episodes with a 1,500-step rolling context window.

\begin{figure}[htbp]
\centering
\includegraphics[width=0.5\linewidth]{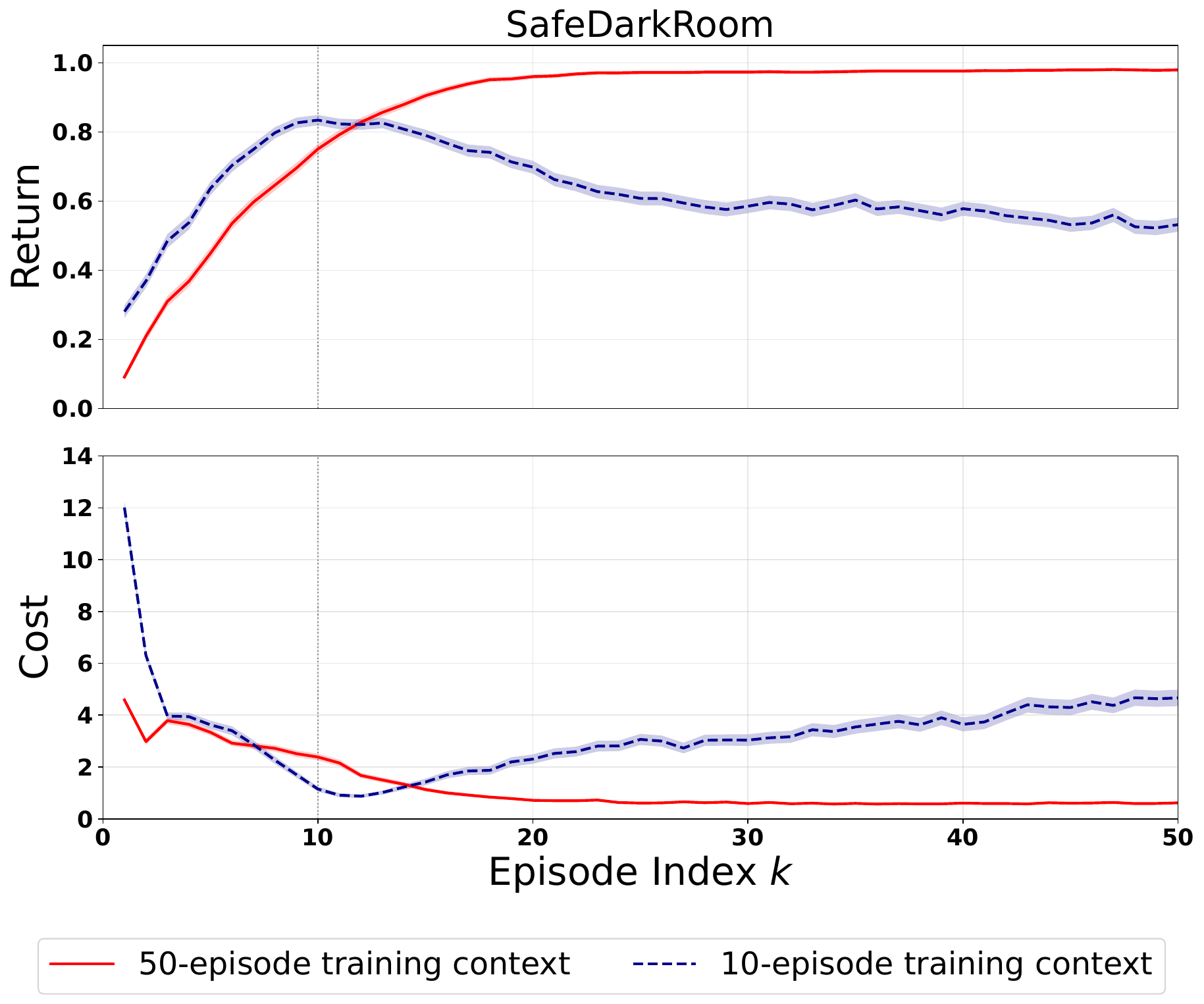}
\caption{\textbf{Training-context ablation on SafeDarkRoom.}
Per-episode return (top) and cost (bottom) for the two training context lengths.
Shaded regions denote standard errors across rollouts. The dotted line marks
$k=10$, the approximate 10-episode equivalent of the shorter sequence at the
30-step episode cap.}
\label{fig:darkroom_training_context_ablation}
\end{figure}
The 300-step model performs better at $k=10$ (return 0.834, cost 1.148,
versus 0.751 and 2.384), but degrades substantially later in evaluation.
At $k=50$, it achieves a return of 0.532 and a cost of 4.664, against 0.980
and 0.617 for the 1,500-step model. These results highlight a limitation of
context-length extrapolation: the short-context model's strong early adaptation
does not persist when evaluation uses histories longer than its training context.

\subsection{Extended Adaptation on Point and Car}
\label{app:pointcar_extended_adaptation}

The Point and Car policies use 1,500-step training context sequences
(approximately 20 episodes).
We run the frozen policies learnt by Q-Barrier for 100 episodes with the same 1,500-step window,
without any retraining. Each episode has at most 75 policy steps, and the cost
budget resets after every episode.

\begin{figure}[htbp]
\centering
\includegraphics[width=\linewidth]{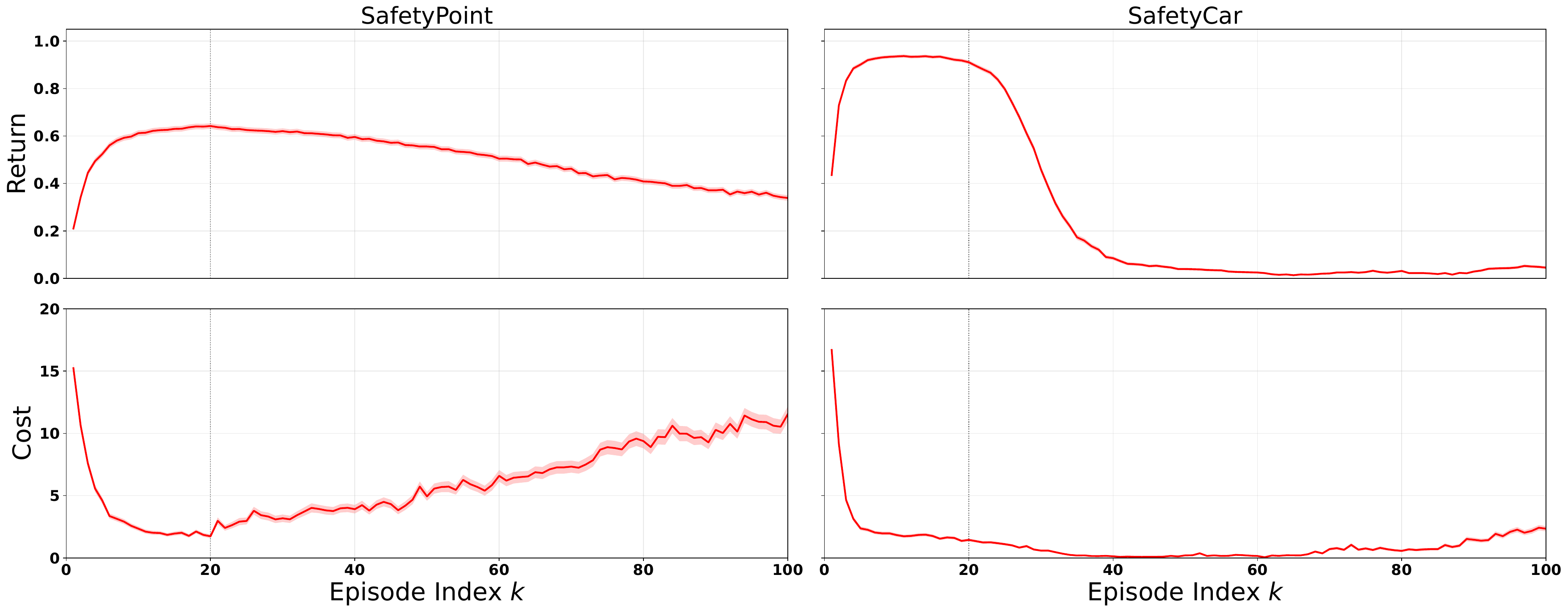}
\caption{\textbf{Extended adaptation on Point and Car.}
Per-episode return (top) and cost (bottom) over 100 episodes. Shaded regions
denote standard errors across rollouts. The dotted line marks $k=20$, the
approximate episode coverage of the training context window.}
\label{fig:pointcar_extended_adaptation}
\end{figure}

On Point, return drops and cost rises from $k=20$ (0.642, 1.736) to $k=100$
(0.339, 11.546). On Car, return and cost are 0.911 and 1.445 at $k=20$, 0.039
and 0.203 at $k=50$, and 0.045 and 2.343 at $k=100$.
Performance degrades substantially as adaptation extends beyond the training
context's approximate 20-episode coverage, with a particularly severe return
collapse on Car. This occurs even with the rolling window fixed at 1,500 steps:
keeping the input within the training context length does not by itself ensure
sustained performance during longer deployment.
 \end{document}